\documentclass[a4paper,fleqn]{cas-dc}

\usepackage[numbers]{natbib}

\usepackage{algorithmic}
\usepackage{algorithm}
\usepackage{array}
\usepackage{textcomp}
\usepackage{stfloats}
\usepackage{url}
\usepackage{bbm}
\usepackage{verbatim}
\usepackage{graphicx}
\usepackage{amsmath,amssymb,amsthm,amsfonts}
\usepackage{graphicx}
\usepackage{subcaption}
\usepackage{float}

\newcommand{\lms}{\{\!\!\{}
\newcommand{\rms}{\}\!\!\}}

\newcommand{\rvline}{\hspace*{-\arraycolsep}\vline\hspace*{-\arraycolsep}}

\newtheorem{theorem}{Theorem}[section]
\newtheorem{lemma}[theorem]{Lemma}

\newtheorem{corollary}[theorem]{Corollary}

\theoremstyle{definition}

\newtheorem{remark}[theorem]{Remark}

\usepackage{enumitem}
\def\tsc#1{\csdef{#1}{\textsc{\lowercase{#1}}\xspace}}
\tsc{WGM}
\tsc{QE}

\begin{document}
\let\WriteBookmarks\relax
\def\floatpagepagefraction{1}
\def\textpagefraction{.001}

\shorttitle{Generalization Limits of GNNs in Identity Effects Learning}    

\shortauthors{D'Inverno et al.}  

\title [mode = title]{Generalization Limits of Graph Neural Networks in Identity Effects Learning}  



%

\author[1]{Giuseppe Alessio D'Inverno}[orcid=0000-0001-7367-4354]
\cormark[1]
\fnmark[1]

\ead{dinverno@diism.unisi.it}

\ead[url]{https://www3.diism.unisi.it/~dinverno/}


\author[2]{Simone Brugiapaglia}[orcid=0000-0003-1927-8232]

\ead{simone.brugiapaglia@concordia.ca}

\ead[url]{http://www.simonebrugiapaglia.ca/}


\author[3,4]{Mirco Ravanelli}[orcid=0000-0002-3929-5526]

\ead{mirco.ravanelli@gmail.com}

\ead[url]{https://sites.google.com/site/mircoravanelli/}


\affiliation[1]{organization={DIISM - University of Siena},
            addressline={via Roma 56}, 
            city={Siena},
            postcode={53100}, 
            state={},
            country={Italy}}
\affiliation[2]{organization={Department of Mathematics and Statistics, Concordia University},
            addressline={1400 De Maisonneuve Blvd. W.}, 
            city={Montréal},
            postcode={H3G 1M8}, 
            state={QC },
            country={Canada}}

\affiliation[3]{organization={Department of Computer Science and Software Engineering, Concordia University},
            addressline={2155 Guy St.}, 
            city={Montréal},
            postcode={H3H 2L9}, 
            state={QC},
            country={Canada}}
\affiliation[4]{organization={Mila - Quebec AI Institute},
            addressline={6666 Saint-Urbain R.}, 
            city={Montréal},
            postcode={H2S 3H1}, 
            state={QC},
            country={Canada}}


\begin{abstract}
Graph Neural Networks (GNNs) have emerged as a powerful tool for data-driven learning on various graph domains. They are usually based on a message-passing mechanism and have gained increasing popularity for their intuitive formulation, which is closely linked to the Weisfeiler-Lehman (WL) test for graph isomorphism to which they have been proven equivalent in terms of expressive power.
In this work, we establish new generalization properties and fundamental limits of GNNs in the context of learning so-called identity effects, i.e., the task of determining whether an object is composed of two identical components or not. Our study is motivated by the need to understand the capabilities of GNNs when performing simple cognitive tasks, with potential applications in computational linguistics and chemistry. 
We analyze two case studies: (i) two-letters words, 
for which we show that GNNs trained via stochastic gradient descent are unable to generalize to unseen letters when utilizing orthogonal encodings like one-hot representations; 
(ii) dicyclic graphs, i.e., graphs composed of two cycles, for which we present positive existence results leveraging the connection between GNNs and the WL test.
Our theoretical analysis is supported by an extensive numerical study. 
\end{abstract}



\maketitle
\section{Introduction}
\label{sec:intro}
Graph Neural Networks (GNNs) \cite{Sca+2009} have emerged as prominent models for handling structured data, quickly becoming dominant in data-driven learning over several scenarios such as network analysis \cite{fan2020graph}, molecule prediction \cite{wieder2020compact} and generation \cite{bongini2021molecular}, text classification \cite{malekzadeh2021review}, and traffic forecasting \cite{jiang2022graph}. From the appearance of the earliest GNN model \cite{Sca+2009}, many variants have been developed to improve their prediction accuracy and generalization power. Notable examples include GraphSage \cite{hamilton2017inductive}, Graph Attention Networks\cite{velivckovic2017graph}, Graph Convolutional Networks (GCN) \cite{kipf2016semi}, Graph Isomorphism Networks\cite{xu2018powerful}, and Graph Neural Diffusion (GRAND) \cite{chamberlain2021grand}.  Furthermore, as the original model was designed specifically for labeled undirected graphs \cite{scarselli2008graph}, more complex neural architectures have been designed to handle different types of graph structures, such as directed graphs \cite{shi2019skeleton}, temporal graphs \cite{longa2023graph}, and hypergraphs \cite{zhang2019hyper}. For a comprehensive review see, e.g., \cite{hamilton2020graph}. 
Over the last decade, there has been  growing attention in the theoretical analysis of GNNs. While approximation properties have been examined in different flavors \cite{azizian2020expressive,d2021new, keriven2019universal,Sca+2009}, most of the theoretical works in the literature have focused on the \textit{expressive power} of GNNs. From this perspective, the pioneering work of \cite{xu2018powerful} and \cite{morris2019weisfeiler} laid the foundation for the standard analysis of GNN expressivity, linking the message-passing iterative algorithm (common to most GNN architectures) to the \textit{first order Weisfeiler Lehman (1--WL) test} \cite{leman1968}, a popular coloring algorithm used to determine if two graphs are (possibly) isomorphic or not. Since then, the expressive power of GNNs has been evaluated with respect to the 1--WL test or its higher-order variants (called \textit{$k$--WL test}) \cite{morris2019weisfeiler}, as well as other variants suited to detect particular substructures \cite{bodnar2021weisfeilera,bodnar2021weisfeilerb}.
 The assessment of the \textit{generalization capabilities} of neural networks has always been crucial for the development of efficient learning algorithms. Several complexity measures have been proposed over the past few decades to establish reliable generalization bounds, such as the Vapnik--Chervonenkis (VC) dimension \cite{vapnik1994measuring}, Rademacher complexity \cite{bartlett2002rademacher,golowich2018size}, and Betti numbers \cite{bianchini2014complexity}. The generalization properties of GNNs have been investigated using these measures. In \cite{scarselli2018vapnik} the VC dimension of the original GNN model was established, and later extended to message passing-based GNNs by \cite{morris2023wl} for piecewise polynomial activation functions. Other generalization bounds for GNNs were derived using the Rademacher complexity \cite{garg2020generalization}, through a Probably Approximately Correct (PAC) Bayesian approach \cite{liao2020pac} or using random sampling on the graph nodes \cite{maskey2022generalization}.
 
An alternative approach for assessing the generalization capabilities of neural networks is based on investigating their ability to learn specific \textit{cognitive tasks} \cite{marcus1999rule,marcus2003algebraic, suarez2021learning}, which have long been of primary interest as neural networks were originally designed to emulate functional brain activities. Among the various cognitive tasks, the linguistics community has shown particular interest in investigating so-called \textit{identity effects}, i.e., the task of determining whether objects are formed by two identical components or not \cite{benua1995identity,gallagher2013learning}. 
To provide a simple and illustrative example, we can consider an experiment in which the words $\mathsf{AA},\mathsf{BB},\mathsf{CC}$ are assigned to the label ``good'', while $\mathsf{AB},\mathsf{BC},\mathsf{AC}$ are labelled as ``bad''. Now, imagine a scenario where a subject is presented with new test words, such as $\mathsf{XX}$ or $\mathsf{XY}$.  Thanks to the human ability of abstraction, the subject will be immediately able to classify the new words correctly, even though the letters $\mathsf{X}$ and $\mathsf{Y}$ were not part of the the training set. Identity effects learning finds other examples in, for instance, \textit{reduplication} (which happens when words are inflected by repeating all or a portion of the word) \cite{paschen2021trigger} or \textit{contrastive reduplication} \cite{ghomeshi2004contrastive}.
Besides their relevance in linguistics, the analysis of identity effects can serve as an intuitive and effective tool to evaluate the generalization capabilities of neural networks in a variety of specific tasks.  These tasks encompass the identification of equal patterns in natural language processing  
 \cite{wu2010natural} as well as molecule classification or regression \cite{wieder2020compact}. In the context of molecule analysis, the exploitation of molecular symmetries as in, for instance, the class of \textit{bicyclic compounds)~\cite{liebman1976survey}} plays a crucial role as it can be exploited to retrieve molecular orientations \cite{bunker2006molecular} or to determine properties of molecular positioning \cite{pettinari2017ir}. Furthermore, the existence of different symmetries in interacting molecules can lead to different reactions. 
 
Recently, it has been shown in \cite{brugiapaglia2022invariance} that Multilayer Perceptrons (MLPs) and Recurrent Neural Networks (RNNs) cannot learn identity effects via Stochastic Gradient Descent nor Adam, under certain conditions on the encoding utilized to represent the components of objects. This finding, based on a framework introduced in \cite{tupper2016learning}, raises a fundamental question that forms the core focus of our paper: ``\textit{Do GNNs possess the capability to learn identity effects?}''
Motivated by this research question, this work  investigates the generalization limits and capabilities of GNNs when learning identity effects. Our contributions are the following:
\begin{itemize}
    \item [(i)] extending the analysis of \cite{brugiapaglia2022invariance}, GNNs are shown to be \emph{incapable} of learning identity effects via SGD training under sufficient conditions determined by the existence of a suitable transformation $\tau$ of the input space 
    (Theorem \ref{th:gnn_rating}); 
    an application to the problem of classifying identical two-letter words is provided by Theorem~\ref{th:rating_alphabet} and supported by numerical experiments in \S\ref{subsec:exp_twoletter};
    \item [(ii)] on the other hand, GNNs are shown to be \textit{capable} of learning identity effects in terms of binary classification of \textit{dicyclic graphs}, i.e., graphs composed by two cycles of different or equal length, in Corollary~\ref{cor:gnn_topological_identity}; a numerical investigation of the gap between our theoretical results and practical performance of GNNs is provided in \S\ref{subsec:exp_dicyclic}.
\end{itemize}
The paper is structured as follows. \S\ref{sec:notation} begins by providing a brief overview of fundamental graph theory notation. We then introduce the specific GNN formulation we focus on in our analysis, namely the Weisfeiler-Lehman test, and revisit the framework of rating impossibility theorems for invariant learners. In \S\ref{sec:main_results}, we present and prove our main theoretical results. \S\ref{sec:experiments} showcases the numerical experiments conducted to validate our findings. Finally, in \S\ref{sec:conclusions}, we provide concluding remarks and outline potential avenues for future research. 

\section{Notation and background}\label{sec:notation}
We start by introducing the notation and background concepts that will be used throughout the paper.
\subsection{Graph theory basics}
A node-attributed graph $G$ is an object defined by a triplet $G = (V, E, \alpha)$. $V$ is the set of \textit{nodes} or \textit{vertices} $v$, where $v$ can be identified as an element of $\mathbb{N} :=\{0,1,2,\ldots\}$. $E$ is the set of edges $e_{u,v}$, where $e_{u,v} = (u,v) \in V \times V$. The term $\alpha: V \rightarrow \mathbb{R}^k$ is the function assigning a \textit{node feature} (or \textit{vertex feature}) $\alpha_v$ to every node $v$ in the graph, with $k$ being the feature dimension. The \textit{number of nodes} of a graph G is denoted by $N:= |V|$. All node features can be stacked in a \textit{feature matrix} $\mathbf{X}_G \in  \mathbb{R}^{N \times k }$. The \textit{adjacency matrix} $\mathbf{A}$ is defined as $A_{ij} = 1$ if $e_{ij} \in E$, $A_{ij} = 0$ otherwise. The \textit{neighborhood} of a node $v$ is denoted by $\mathcal{N}_v = \{ u \;| \;e_{u,v} \in E   \}$. 
\subsection{Graph Neural Networks}
Graph Neural Networks (GNNs) are a class of connectionist models that aim to learn functions on graphs, or pairs graph/node. 
Intuitively, a GNN learns how to represent the nodes of a graph by vectorial representations (which are called \textit{hidden states}), giving an encoding of the information stored in the graph. In its general form \cite{Gil+2017,Sca+2009}, for each graph $G=(V,E,\alpha) \in \mathcal{G}$ where $\mathcal{G}$ is a node-attributed graph domain, a GNN is defined by the following recursive \textit{updating scheme}:
\begin{equation}\label{def:gnn_upd}
    h_v^{(t+1)} = \text{UPDATE}^{(t+1)}\big(  h_v^{(t)}, \text{AGGREGATE}^{(t+1)} (\lms h_u^{(t)} | u \in \mathcal{N}_v \rms)   \big), 
\end{equation}
for all $v\in V$ and $t =1 , \dots, T$, where $h_v^{(t)}$ is the hidden feature of node $v$ at time $t$, $T$ is the number of layers of the GNN and $\lms \cdot \rms $ denotes a multiset.
Here $\{ \text{UPDATE}^{(t)} \}_{t=1,\dots, T}$ and $\{ \text{AGGREGATE}^{(t)} \}_{t=1,\dots, T}$ are families of functions that can be defined by learnable or non-learnable schemes. Popular GNN models like GraphSAGE \cite{hamilton2017inductive}, GCN \cite{kipf2016semi},  Graph Isomorphism Networks \cite{xu2018powerful} are based on this updating scheme. 
The model terminates with a $\text{READOUT}$ function, chosen according to the nature of the task; for instance, global average, min or sum pooling, followed by a trainable multilayer perceptron are typical choices in the case of graph-focused tasks. At a high level, we can formalizate a GNN as a function $g: \mathcal{G} \rightarrow \mathbb{R}^r$, where $\mathcal{G}$ is a set of node-attributed graphs and $r$ is the dimension of the output, which depends on the type of task at hand.
The updating scheme we choose as a reference for our analysis follows \cite{morris2019weisfeiler}.
This model has been proven to match the expressive power of the Weisfeiler-Lehman test \cite{morris2019weisfeiler} (see also Theorem~\ref{th:wl_morris} below), and can therefore be considered a good representative model of the message passing GNN class. The hidden feature $h_v^{(t+1)} \in \mathbb{R}^h$ of a node $v$ at the message passing iteration $t+1$, for $t=1, \dots, T-1$, is defined as
\begin{equation}\label{morris_gconv}
    h_v^{(t+1)} = \sigma \big(W^{(t+1)}_{\text{upd}} h^{(t)}_v + W^{(t+1)}_{\text{agg}}  h^{(t)}_{\mathcal{N}_v} + b^{(t+1)} \big ), 
\end{equation}
where $h^{(t)}_{\mathcal{N}_v} = \text{POOL} \lms h^{(t)}_u | u \in \mathcal{N}_v \rms$ , $\sigma: \mathbb{R}^h \rightarrow \mathbb{R}^h$ is an element-wise activation function and  POOL is the aggregating operator on the neighbor node's features. The aggregating operator can be defined as a non-learnable function,  such as the sum, the mean or the minimum, across the hidden features of the neighbors. 
With respect to equation \eqref{def:gnn_upd}, we have that $\text{AGGREGATE}^{(t)}(\cdot) = \text{POOL}(\cdot)$ $\forall t=1, \dots, T$, while $\text{UPDATE}^{(t+1)}(h_v, h_{\mathcal{N}_v}) = \sigma \big(W^{(t+1)}_{\text{upd}} h_v + W^{(t+1)}_{\text{agg}}  h_{\mathcal{N}_v} + b^{(t+1)} \big ). $
For each node, the initial hidden state is initialized as $h_v^{(0)} = \alpha_v \in \mathbb{R}^k$.  
The learnable parameters of the GNN can be summarized as $\Theta := (W^{(0)}_{\text{upd}}, W^{(0)}_{\text{agg}}, b^{(0)}, W^{(1)}_{\text{upd}}, W^{(1)}_{\text{agg}}, b^{(1)}, \dots,  W^{(L)}_{\text{upd}}, W^{(L)}_{\text{agg}}, b^{(L)})$, with $W^{(0)}_{\text{upd}}, W^{(0)}_{\text{agg}} \in \mathbb{R}^{k \times h}$, $W^{(t)}_{\text{upd}}, W^{(t)}_{\text{agg}} \in \mathbb{R}^{h\times h}$, for $t=1, \dots, T$, and $b^{(t)} \in \mathbb{R}^h$, for $t=0, \dots, T$. 
\subsection{The Weisfeiler--Lehman test}
The  \textit{first order Weisfeiler--Lehman test}  (in short, \textit{1--WL test}) \cite{leman1968} is one of the most popular isomorphism tests for graphs, based on an iterative coloring scheme. 
The coloring algorithm is applied in parallel to two input graphs, giving a color partition of the nodes as output. If the partitions match, then the graphs are possibly isomorphic, while if they do not match, then the graphs are certainly non--isomorphic. Note that the test is not conclusive in the case of a positive answer, as the graphs may still be non--isomorphic; nevertheless, the 1--WL test provides an accurate isomorphism test for a large class of graphs \cite{babai1979canonical}.
The coloring is carried out by an iterative algorithm which takes as input a graph $G = (V,E,\alpha)$ and,  at each iteration, computes a \textit{node coloring} $c^{(t)}(v) \in \mathcal{C}$ for each node $v \in V$, being $\mathcal{C} \subseteq \mathbb{N}$ a subset of natural numbers representing colors.  The algorithm is sketched in the following.
\begin{enumerate}
\item At iteration 0, in the case of labeled graphs, the node color initialization is based on the vertex feature according to a specific hash function $\text{HASH}_0 : \mathbb{R}^k \rightarrow \mathcal{C}$; namely, $c^{(0)}(v)=\text{HASH}_0(\alpha(v))$, for all $v \in V$.
For unlabeled graphs, a node color initialization is provided, usually setting every color as equal to a given initial color $c^{(0)} \in \mathcal{C}$.

\item
For any iteration $t >0$, we set
\begin{equation*}
c ^{(t)}(v)=\text{HASH}((c^{(t-1)}(v),\lms c^{(t-1)}(n)| n \in \mathcal{N}_v \rms )), 
\end{equation*}
$\forall v \in V$, where $\text{HASH}$ injectively maps the above color-multiset pair to a unique value in $\mathcal{C}$. 
\end{enumerate}
The algorithm terminates if the number of colors between two iterations does not change, i.e., when the cardinalities of $\{ c^{(t-1)}(v)| v \in V\} $ and $\{ c^{(t)}(v)| v \in V\} $, namely,  are equal.

We conclude by recalling two results establishing the equivalence between GNNs' and 1--WL test's expressive power that will be instrumental for our analysis. A first result was proved in \cite{xu2018powerful} and it characterizes the equivalence on a graph-level task for GNNs with generic message passing layers satisfying suitable conditions.
Another characterization, reported below, is due to \cite{morris2019weisfeiler} and states the equivalence on a node coloring level, referring to the particular model defined in \eqref{morris_gconv}.
 \begin{theorem}[See {\cite[Theorem 2]{morris2019weisfeiler}}]\label{th:wl_morris}
  Let $G=(V,E,\alpha)$ be a graph with initial coloring $c^{(0)}(v)\in \mathbb{R}$ for each node $v \in V$ (so that $c^{(0)} \in \mathbb{R}^{|V(G)|}$). Then, for all $t \geq 0$ there exists a GNN of the form \eqref{morris_gconv} such that the hidden feature vector $h^{(t)} \in \mathbb{R}^{|V(G)|}$  produced by the GNN at layer $t$ coincides with the color vector $c^{(t)}\in \mathbb{R}^{|V(G)|}$ produced by the 1--WL test at iteration $t$, i.e.,  $c^{(t)} \equiv h^{(t)}$.
    \end{theorem}
\subsection{Rating impossibility for invariant learners}
We now recall the framework of rating impossibility from \cite{brugiapaglia2022invariance}, which we will then apply to the case of identity effects learning. In general, we assume to train a \textit{learning algorithm} to perform a rating assignment task, where the rating $r$ is a real number.  Let $\mathcal{I}$ be the set of all possible inputs $x$ (that could be, for instance, elements of $\mathbb{R}^d$). Our learning algorithm is trained on a dataset $D \subseteq \mathcal{I} \times \mathbb{R}$ consisting of a finite set of input-rating pairs $(x,r)$ . Let $\mathcal{D}$ be the set of all possible datasets with inputs in $\mathcal{I}$.
The learning algorithm is trained via a suitable optimization method, such as Stochastic Gradient Descent (SGD) or Adaptive Moment Estimation (Adam) \cite{kingma2014adam}, that for any given training dataset $D$ outputs the optimized set of parameters $\Theta = \Theta(D) \in \mathbb{R}^p$, which, in turn, defines a model $f = f(\Theta, \cdot)$. 
The rating prediction on a novel input $x \in \mathcal{I}$ is then given by
$r = f(\Theta, x)$.
In summary, a learning algorithm can thought of as a map $L: \mathcal{D} \times \mathcal{I} \rightarrow \mathbb{R}$, defined as $L(D,x) = f(\Theta(D),x)$.

Given the stochastic nature of neural network training, we adopt a nondeterministic point of view. Hence we require the notion of \textit{equality in distribution}.
Two random variables $X,Y$ taking values in $\mathbb{R}^k$ are said to be \textit{equal in distribution} (denoted by $X \overset{d}{=} Y$) if $\mathbb{P} (X \leq x) = \mathbb{P} (Y \leq x)$ for all $x \in \mathbb{R}^k$, where inequalities hold componentwise.
With this notation, rating impossibility means that 
$L(D,x_1) \overset{d}{=} L(D,x_2)$ 
for two inputs $x_1 \neq x_2$ drawn from $\mathcal{I} \setminus D$. 
Sufficient conditions for rating impossibility are identified by the following theorem from \cite{brugiapaglia2022invariance} (here slightly adapted using equality in distribution), which involves the existence of an auxiliary transformation $\tau$ of the inputs.
\begin{theorem}[Rating impossibility for invariant learners, {\cite[Theorem~1]{brugiapaglia2022invariance}}]\label{th:rating_old}
    Consider a dataset $D\subseteq \mathcal{I} \times \mathbb{R}$ and a transformation $\tau: \mathcal{I} \rightarrow \mathcal{I}$ such that 
    \begin{itemize}
        \item [(i)] $\tau (D) \overset{d}{=} D$ (invariance of the data).\footnote{By definition, $\tau(D):=\{(\tau(x), r) : (x,r) \in D\}$.}
    \end{itemize} 
    Then, for any learning algorithm $L: \mathcal{D} \times \mathcal{I}\rightarrow \mathbb{R} $ and any input $x \in \mathcal{I}$ such that
    \begin{itemize}
        \item [(ii)] $L(\tau(D),\tau(x)) \overset{d}{=} L(D,x)$  (invariance of the algorithm),
    \end{itemize}
    we have $L(D,\tau(x)) \overset{d}{=} L(D,x)$.
\end{theorem}
This theorem states that under the invariance of the data and of the algorithm, the learner cannot assign different ratings to an input $x$ and its transformed version $\tau(x)$. This leads to rating impossibility when $\tau(x) \neq x$ and $x, \tau(x) \in \mathcal{I}\setminus D$.

We conclude by recalling some basic notions on SGD training. Given a dataset $D$, we aim to find parameters $\Theta$ that minimize an objective  function of the form
$$
F(\Theta) = \mathcal{L}((f(\Theta, x), r) : (x,r) \in D), \quad \Theta \in \mathbb{R}^p,
$$
where $\mathcal{L}$ is a (possibly regularized) loss function.
We assume $F$ to be differentiable over $\mathbb{R}^p$ in order for its gradients to be well defined. Given a collection of subsets $(D_i)_{i=0}^{k-1}$ with $D_i \subseteq D$ (usually referred to as training batches, which can be either deterministically or randomly generated), we define $F_{D_i}$ as the function $F$ where the loss is evaluated only on data in $D_i$. In SGD-based training, we randomly initialize $\Theta_0$ and iteratively compute
\begin{equation}\label{eq:SGD}
    \Theta_{i+1} = \Theta_i - \eta_i \frac{\partial F_{D_i}}{\partial \Theta} (\Theta_i), 
\end{equation}
for $i = 0,1, \dots, k-1$, where the sequence of step sizes $(\eta_i)_{i=0}^{k-1}$ is assumed to be either deterministic or random and independent of $(D_i)_{i=0}^{k-1}$.  Note that, being $\Theta_i$ a random vector for each $i$, the output of the learning algorithm $L(D,x) = f(\Theta_k, x)$ is a random variable. 
\section{Theoretical analysis}\label{sec:main_results}

In this section we present our theoretical analysis. More specifically, in \S\ref{subsec:id_eff_words} we establish a rating impossibility theorem for GNNs under certain technical assumptions related to the invariance of the training data under a suitable transformation $\tau$ of the inputs; then, we illustrate an application to the case study of identity effects learning for a two-letter word dataset in \S\ref{subsubsec:id_eff_application}. In \S\ref{subsec:id_eff_cycles} we prove that symmetric dicyclic graphs can be distinguished from the asymmetric ones by the 1--WL test, and consequently by a GNN. 

\subsection{What GNNs cannot learn:  rating impossibility theorem}\label{subsec:id_eff_words}


We assume the input space to be of the form $\mathcal{I}=\mathbb{R}^{d} \times \mathbb{R}^{d}$ and the learning algorithm 
\begin{equation} \label{learning_algorithm_gnn}
    L(D,x) = f(B, Gu + Hv, Hu + Gv), \; \forall x = (u,v) \in \mathcal{I},
\end{equation}
where $\Theta = (B,G,H)$ are trainable parameters and $G, H\in \mathbb{R}^{d \times d}$.
This class of learning algorithms perfectly fits the formulation given in \cite{morris2019weisfeiler}, where the updating scheme is the one defined by \eqref{morris_gconv}. In this case, 
\begin{align*}
G &= W^{(1)}_{upd}, \quad  H = W^{(1)}_{agg},\\
B &= \left (b^{(1)}, W^{(2)}_{upd}, W^{(2)}_{agg}, b^{(2)} \dots, W^{(N)}_{upd}, W^{(N)}_{agg}, b^{(N)}  \right).
\end{align*}
The learner defined by equation \eqref{learning_algorithm_gnn}  mimics, in this specific setting, the behaviour of several GNN architectures, GCN included. In fact, when the graph is composed by only two nodes, the convolution ends up being a weighted sum of the hidden states of the two nodes, i.e.,  $h_{\mathcal{N}_t}^{(t)} = h_u^{(t)}$ and 
\begin{equation*}
    h_v^{(t+1)} = \sigma \big(W^{(t+1)}_{\text{upd}} h^{(t)}_v + W^{(t+1)}_{\text{agg}}  h^{(t)}_u + b^{(t+1)} \big ).
\end{equation*}
This property will have practical relevance in Theorem~\ref{th:rating_alphabet} and its experimental realization in \S\ref{subsec:exp_twoletter}. 

In the following result we identify sufficient conditions on the dataset $D$ and the training procedure able to guarantee invariance of GNN-type models of the form \eqref{learning_algorithm_gnn} trained via SGD to a suitable class of transformations $\tau$ (hence verifying condition (ii) of Theorem~\ref{th:rating_old}).
\begin{theorem}[Invariance of GNN-type models trained via SGD] \label{th:gnn_rating}
Assume the input space to be of the form $ \mathcal{I} = \mathbb{R}^d \times \mathbb{R}^d$. Let $\tau: \mathcal{I} \rightarrow \mathcal{I}$ be a linear transformation defined by $\tau(x) = (u, \tau_2(v))$  for any $x = (u,v) \in \mathcal{I}$, where $\tau_2 : \mathbb{R}^d \rightarrow \mathbb{R}^d$ is also linear. Moreover, assume that
\begin{itemize}
    \item the matrix $T_2 \in \mathbb{R}^{d\times d}$ associated with the transformation $\tau_2$ is orthogonal and symmetric; 
    \item the dataset $D = \{ ( ( u_i, v_i), r_i) \}_{i=1}^n$ is invariant under the transformation $\tau_2 \otimes \tau_2$, i.e.,
    \begin{equation}\label{input_inv}
        (u_i, v_i) = \big ( \tau_2 (u_i), \tau_2 (v_i) \big ), \quad \forall i=1, \ldots, n.
    \end{equation}
\end{itemize}
Suppose k iterations of SGD as defined in \eqref{eq:SGD} are used to determine parameters $\Theta_k = ( B_k, G_k, H_k)$ with objective function
\begin{equation*}
    F(\Theta) = \sum \limits_{i=1}^n \ell\big (f(B, Gu_i + Hv_i, Hu_i + Gv_i), r_i \big ) + \lambda \mathcal{R} (B),
\end{equation*}
for some $\lambda \geq 0$, with $\Theta = (B,G,H)$ and where $\ell$, $f$ and $\mathcal{R}$ are real-valued functions such that $F$ is differentiable.
Suppose the random initialization of the parameters $B$, $G$ and $H$ to be independent and that the distributions of $G_0$ and $H_0$ are invariant with respect to right-multiplication by $T_2$. 
Then, the learner $L$  defined by $L(D,x) = f(B_k, G_k u + H_k v, H_k u + G_k v)$, for $x = (u,v)$, satisfies $L(D,x) \overset{d}{=}  L(\tau (D), \tau (x))$.
\end{theorem}

\begin{proof}
Given a batch $D_i \subseteq D$, define $J_i := \{j \in \{ 1, \ldots, n\} : ((u_j,v_j), r_j) \in D_i \}$ and
\begin{align*}
   F_{D_i} (\Theta) = 
   \sum_{j\in J_i} \ell(f(B, Gv_j + Hu_j, Hv_j + Gu_j), r_j) + \lambda \mathcal{R}(B).
\end{align*}
Moreover, consider an \textit{auxiliary objective function}, defined by
\begin{align*}
    & \Tilde{F}_{D_i} (B, G_1, H_1, H_2, G_2) = \\&\quad \sum_{j\in J_i} \ell(f(B, G_1v_j + H_1u_j, H_2v_j + G_2u_j), r_j) + \lambda \mathcal{R} (B).
\end{align*}
Observe that $F_{D_i} (\Theta) = \Tilde{F}_{D_i}(B, G, H, H, G)$. Moreover, 
\begin{align}
\label{aux_B}
    \frac{\partial F_{D_i}}{\partial B}(\Theta) & = \frac{\partial \Tilde{F}_{D_i}}{\partial B} (\Theta)\\
\label{aux_G}
    \frac{\partial F_{D_i}}{\partial G}(\Theta) &= \frac{\partial \Tilde{F}_{D_i}}{\partial G_1} (\Theta) + \frac{\partial \Tilde{F}_{D_i}}{\partial G_2} (\Theta)\\
\label{aux_H}
    \frac{\partial F_{D_i}}{\partial H}(\Theta) &= \frac{\partial \Tilde{F}_{D_i}}{\partial H_1} (\Theta) + \frac{\partial \Tilde{F}_{D_i}}{\partial H_2} (\Theta)
\end{align}
Moreover, replacing $D_i$ with its transformed version $\tau(D_i) = \{((u_j, \tau_2 (v_j)), r_j ) \}_{j \in D_i}$, we see that $F_{\tau(D_i)}(\Theta) = \Tilde{F}_{D_i}(B, G, HT_2, H, G T_2 ) $. 
This leads to
\begin{align}
\label{tau_B}
    \frac{\partial F_{\tau (D_i)}}{\partial B} (\Theta) &= \frac{\partial \Tilde{F}_{D_i}}{\partial B} (B, G, HT_2, H, G T_2 )\\
    \nonumber
    \frac{\partial F_{\tau(D_i)}}{\partial G}(\Theta) & = \frac{\partial \Tilde{F}_{D_i}}{\partial G_1} (B, G, HT_2, H, G T_2 ) \\
\label{tau_G}
     &\quad + \frac{\partial \Tilde{F}_{D_i}}{\partial G_2} (B, G, HT_2, H, G T_2 )T_2^T \\
\nonumber
    \frac{\partial F_{\tau(D_i)}}{\partial H}(\Theta) &= \frac{\partial \Tilde{F}_{D_i}}{\partial H_1} (B, G, HT_2, H, G T_2 )T_2^T  \\
    \label{tau_H}
    & \quad +\frac{\partial \Tilde{F}_{D_i}}{\partial H_2} (B, G, HT_2, H, G T_2 ).
\end{align}
Now, denoting $\ell = \ell (f,r)$ and $f = f(B, u, v)$, we have 
\begin{align*}
    \frac{\partial \Tilde{F}_{D_i}}{\partial G_1} &= \sum \limits_{j \in D_i} \frac{\partial \ell}{\partial f} \frac{\partial f}{\partial u} v_j^T,
    \quad 
    \frac{\partial \Tilde{F}_{D_i}}{\partial H_1} = \sum \limits_{j \in D_i} \frac{\partial \ell}{\partial f} \frac{\partial f}{\partial u} u_j^T,\\    
    \frac{\partial \Tilde{F}_{D_i}}{\partial H_2} &= \sum \limits_{j \in D_i} \frac{\partial \ell}{\partial f} \frac{\partial f}{\partial v} v_j^T,
    \quad 
    \frac{\partial \Tilde{F}_{D_i}}{\partial G_2} = \sum \limits_{j \in D_i} \frac{\partial \ell}{\partial f} \frac{\partial f}{\partial v} u_j^T.
\end{align*}
In addition, thanks to assumption \eqref{input_inv}, we have $u_j^T T_2^T = u_j^T$ and $v_j^T T_2^T = v_j^T$ for all $j \in J_i$. Thus, we obtain
\begin{align} \label{inv_I_J}
    \frac{\partial \Tilde{F}_D}{\partial G_1} T_2^T &=  \frac{\partial \Tilde{F}_D}{\partial G_1}, 
    \quad 
    \frac{\partial \Tilde{F}_D}{\partial H_1} T_2^T =  \frac{\partial \Tilde{F}_D}{\partial H_1},\\
\label{inv_K_L}
    \frac{\partial \Tilde{F}_D}{\partial H_2} T_2^T &=  \frac{\partial \Tilde{F}_D}{\partial H_2}, 
    \quad
    \frac{\partial \Tilde{F}_D}{\partial G_2} T_2^T =  \frac{\partial \Tilde{F}_D}{\partial G_2}.
\end{align}

Now, let $(B_0', G_0', H_0' ) \overset{d}{=} (B_0, G_0 , H_0) $ and let $(B_i', G_i' , H_i')$ for $i = 1, \ldots k$ be the sequence generated by SGD, applied to the transformed data $\tau(D)$. By assumption, we have $B_0' \overset{d}{=}B_0$, $G_0 \overset{d}{=} G_0' \overset{d}{=} G_0' T_2 $ and $H_0 \overset{d}{=} H_0' \overset{d}{=} H_0' T_2 $ . We now show by induction that $B_i' \overset{d}{=} B_i$ , $G_i \overset{d}{=} G_i' \overset{d}{=} G_i' T_2 $ and $H_i \overset{d}{=} H_i' \overset{d}{=} H_i' T_2 $ for all indices $i=1, \ldots, k$. Using equations \eqref{aux_B} and \eqref{tau_B} and the inductive hypothesis, we have
\begin{align*}
    B_{i+1}' & =  B_i' - \eta_i \frac{\partial F_{\tau (D_i)}}{\partial B} (B_i', G_i',H_i') \\
    & = B_i' - \eta_i  \frac{\partial \Tilde{F}_{D_i}}{\partial B} (B_i', G_i', H_i'T_2, H_i', G_i' T_2 ) \\
    & \overset{d}{=} B_i - \eta_i  \frac{\partial \Tilde{F}_{D_i}}{\partial B} (B_i, G_i, H_i, H_i, G_i) \\
    & = B_i - \eta_i \frac{\partial F_{\tau (D_i)}}{\partial B} (B_i, G_i,H_i) = B_{i+1}.
\end{align*}
Similarly, using equations \eqref{aux_G}, \eqref{tau_G} and \eqref{inv_K_L} and the inductive hypothesis, we see that
\begin{align*}
    G_{i+1}' & = G_i' - \eta_i \frac{\partial F_{\tau(D_i)}}{\partial G}(B_i',G_i',H_i') \\
    & =  G_i' -\eta_i \left (\frac{\partial \Tilde{F}_{D_i}}{\partial I} (B_i', G_i', H_i'T_2, H_i', G_i' T_2 )\right. \\
    &  \left.\quad \; +\frac{\partial \Tilde{F}_{D_i}}{\partial L} (B_i', G_i', H_i' T_2, H_i', G_i' T_2 )T_2^T \right ) \\
    & =  G_i' -\eta_i \left (\frac{\partial \Tilde{F}_{D_i}}{\partial G_1} (B_i', G_i', H_i'T_2, H_i', G_i' T_2 )\right. \\ 
    & \quad \left.\;+\frac{\partial \Tilde{F}_{D_i}}{\partial G_2} (B_i', G_i', H_i'T_2, H_i', G_i' T_2 ) \right ) \\
    & \overset{d}{=}  G_i -\eta_i \left (\frac{\partial \Tilde{F}_{D_i}}{\partial G_1} (B_i, G_i, H_i, H_i, G_i  )\right.  \\
    & \left.\quad \; +\frac{\partial \Tilde{F}_{D_i}}{\partial G_2} (B_i, G_i, H_i, H_i, G_i ) \right ) \\
    & =   G_i - \eta_i \frac{\partial F_{D_i}}{\partial G}(B_i,G_i,H_i) = G_{i+1}.
\end{align*}
One proceeds analogously  for $H_{i+1}'$ using equations \eqref{aux_H}, \eqref{tau_H} and \eqref{inv_I_J}. Similarly, one also sees that $G_{i+1}' T_2 \overset{d}{=} G_{i+1}$ and $H_{i+1}' T_2 \overset{d}{=} H_{i+1}$ combining the previous equations with symmetry and orthogonality of $T_2$.

In summary, we have 
\begin{align*}
    L(D,x) & = f(B_k, G_k u + H_k v, H_k u + G_k v) \\
    & \overset{d}{=} f(B_k', G_k' u + H_k' v, H_k' u + G_k' v ) \\
    & \overset{d}{=} f(B_k', G_k' u + H_k'T_2 v, H_k' u + G_k'T_2 v ) \\
    & = L(\tau (D), \tau (x)),
\end{align*}
which concludes the proof. 
\end{proof}

\begin{remark}[On the assumptions of Theorem~\ref{th:gnn_rating}]
    At first glance, the assumptions of Theorem~\ref{th:gnn_rating} might seem quite restrictive, especially the assumption about the invariance of the distributions of $G_0$ and $H_0$ with respect to right-multiplication by the symmetric orthogonal matrix $T_2$. Yet, this hypothesis holds, e.g., when the entries of $G_0$ and $H_0$ are independently and identically distributed according to a centered normal distribution thanks to the rotational invariance of isotropic random Gaussian vectors (see, e.g., \cite[Proposition~3.3.2]{vershynin2018high}). This is the case in common initialization strategies such as Xavier initialization \cite{glorot2010understanding}. In addition, numerical results presented in \S\ref{sec:experiments} suggest that rating impossibility might hold in more general settings, such as when the model $f$ includes ReLU activations (hence, when $F$ has points of nondifferentiability) or for models trained via Adam as opposed to SGD.
\end{remark}

\subsubsection{Application to identity effects}\label{subsubsec:id_eff_application}

As a practical application of Theorem~\ref{th:gnn_rating} to identity effects, we consider the problem of classifying identical two-letter words of the English alphabet $\mathcal{A} := \{\mathsf{A},\mathsf{B},\ldots, \mathsf{Z}\}$, already mentioned in \S\ref{sec:intro} and following \cite{brugiapaglia2022invariance}. Consider a training set $D$ formed by two-letter words that do not contain  $\mathsf{Y}$ nor $\mathsf{Z}$. Words are assigned the label 1 if they are composed by identical letters and 0 otherwise. Our goal is to verify whether a learning algorithm is capable of generalizing this pattern correctly to words containing the letters $\mathsf{Y}$ or $\mathsf{Z}$. 
The transformation $\tau$ of Theorem \ref{th:gnn_rating} is defined by
\begin{equation}
\label{eq:def_tau}
    \tau(\mathsf{xY}) = \mathsf{xZ}, \; \tau(\mathsf{xZ}) = \mathsf{xY}, \; \text{and}\; \tau(\mathsf{xy}) = \mathsf{xy},
\end{equation}
for all letters $\mathsf{x}, \mathsf{y} \in \mathcal{A}$, with $\mathsf{y} \neq \mathsf{Y},\mathsf{Z}$. Note that this transformation is of the form $\tau = I \otimes \tau_2$, where $I$ is the identity map. Hence, it fits the setting of Theorem~\ref{th:gnn_rating}. Moreover, since $D$ does not contain $\mathsf{Y}$ nor $\mathsf{Z}$ letters, $\tau(D) = D$. Hence, condition (i) of Theorem~\ref{th:rating_old} is satisfied. 

In order to represent letters as vectors of $\mathbb{R}^d$, we need to use a suitable \emph{encoding}.  Its choice is crucial to determine the properties of the transformation matrix $T_2$ associated with $\tau_2$, needed to apply Theorem~\ref{th:gnn_rating}. Formally, an encoding of an alphabet $\mathcal{A}$ is a set of vectors $\mathcal{E} \subseteq \mathbb{R}^d$, of the same cardinality of $\mathcal{A}$, to which letters can be associated with. In our case, $|\mathcal{A} |  = 26 = |\mathcal{E}|$. We say that an encoding is \emph{orthogonal} if it is an orthonormal set of $\mathbb{R}^d$. For example, the popular one-hot encoding  $\mathcal{E} = \{e_i\}_{i=1}^{26} \subseteq \mathbb{R}^{26}$, i.e., the canonical basis of $\mathbb{R}^{26}$, is an orthogonal encoding.

In this setting, every word is modeled as a graph defined by two nodes connected by a single unweighted and undirected edge. Each node $v$ is labeled with a node feature $\alpha(v) \in  \mathbb{R}^{d}$, corresponding to a letter's encoding. An example is depicted in Figure~\ref{fig:two-letters-graph}.
 \begin{figure}[t]
     \centering
     \includegraphics[width = 0.45\textwidth]{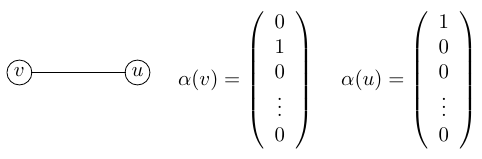}
     \caption{Graph modeling of a two-letter word: a vertex feature $\alpha(v) \in \mathbb{R}^{d}$ is attached to each node $v$ of a two-node undirected graph, according to a given encoding $\mathcal{E}$ of the English alphabet $\mathcal{A}$. In this figure, $\mathcal{E}$ is the one-hot encoding.}
     \label{fig:two-letters-graph}
 \end{figure}
 
\begin{theorem}[Inability of GNNs to classify identical two-letter words outside the training set]\label{th:rating_alphabet}
    Let $\mathcal{E} \subseteq \mathbb{R}^{26}$ be an orthogonal encoding of the English alphabet $\mathcal{A}$ and let $L$ be a learner obtained by training a GNN of the form \eqref{morris_gconv}  via SGD to classify identical two-letter words. Assume that words in the training set $D$ do not contain the letter $\mathsf{Y}$ nor $\mathsf{Z}$. Then, $L$ assigns the same rating (in distribution) to any word of the form $\mathsf{xy}$ where $\mathsf{y} \in \{\mathsf{Y},\mathsf{Z}\}$, i.e., $L(D, \mathsf{xY}) \overset{d}{=} L(D, \mathsf{xZ})$ for any $\mathsf{x} \in \mathcal{A}$. Hence, it is unable to generalize to identity effect outside the training set.
    \end{theorem}

\begin{proof}
As discussed above, the transformation $\tau$ defined by \eqref{eq:def_tau} is of the form $\tau = I \otimes \tau_2$. Moreover, the matrix associated with the linear transformation $\tau_2$ is of the form $T_2 = B^{-1} P B$, where $B$ is the change-of-basis matrix from the orthonormal basis associated with the encoding $\mathcal{E}$ to the canonical basis of $\mathbb{R}^{26}$ (in particular, $B$ is orthogonal and $B^{-1} = B^T$) and $P$ is a permutation matrix that switches the last two entries of a vector, i.e., using block-matrix notation,
\begin{equation*}
    P = \begin{bmatrix}
  I
  & \rvline & 0 \\
\hline
  0 & \rvline &
  \begin{matrix}
  0 & 1 \\
  1 & 0
  \end{matrix}
\end{bmatrix}, \quad I \in \mathbb{R}^{24\times 24}.
\end{equation*}
Hence, $T_2$ is orthogonal and symmetric, and therefore fits the framework of the Theorem \ref{th:gnn_rating}. 

On the other hand, as discussed in \S\ref{subsec:id_eff_words}, every GNN of the form \eqref{morris_gconv} is a model of the form \eqref{learning_algorithm_gnn}. Thus, Theorem~\ref{th:gnn_rating} yields $L(D,\mathsf{xy}) \overset{d}{=}  L(\tau (D), \tau (\mathsf{xy}))$, 
for all letters $\mathsf{x},\mathsf{y}\in \mathcal{A}$. In particular, 
$L(D,\mathsf{xY}) \overset{d}{=}  L(\tau (D), \mathsf{xZ})$, which corresponds to condition (ii) of Theorem~\ref{th:rating_old}. Recalling that $\tau(D) = D$, also condition (i) holds. Hence,  we can apply Theorem~\ref{th:rating_old}
and conclude the proof. 
\end{proof}


\subsection{What GNNs can learn: identity effects on dicyclic graphs }\label{subsec:id_eff_cycles}

We now analyze the expressivity of GNNs to learn identity effects related to the \textit{topology} of the graphs in the dataset. This novel setting requires to design \textit{ex novo} the formulation of our problem. In fact, we are not focusing on the feature matrix $X_G$ of a graph anymore, but on its adjacency matrix $A$, which contains all the topological information.  Here we focus on a particular class of graphs, which we call \emph{dicyclic graphs}. A dicyclic graph is a graph composed by an $m$-cycle and an $n$-cycle, linked by a single edge. Since a dicyclic graph is uniquely determined by the length of the two cycles, we can identify it with the equivalence class $[m,n]$ over the set of pairs $(a,b)$,  $a,b \in \mathbb{N}$, defined as $[m,n] := \{(m,n), (n,m)\}$. A dicyclic graph $[m,n]$ is \emph{symmetric} if $m=n$ and \emph{asymmetric} otherwise.

In this section we provide an analysis of the expressive power of GNNs when learning identity effects on dicyclic graphs (i.e., classifying whether a dicyclic graph is symmetric or not). 
We start by proving a lemma that shows how information propagates through the nodes of a cycle, during the 1--WL test iterations, when one of the nodes has a different initial color with respect to all the other nodes.

\begin{lemma}[1-WL test on $m$-cycles]\label{lem:k_cycles}
Consider an $m$-cycle in which the vertices are numbered from $0$ to $m-1$ clockwise, an initial coloring $c^{(0)} = [0,1, \dots, 1]^T \in \mathbb{N}^m$ (vector indexing begins from 0, and the vector is meant to be circular, i.e., $c^{(0)}(m)=c^{(0)}(0)$), and define the function $\textnormal{HASH}$ as
\begin{equation*}
\begin{cases}
    \textnormal{HASH}(0,\lms j,k \rms) = 0 & \\
    \textnormal{HASH}(i,\lms j,k\rms) = i & \; \text{if} \; j \neq k, \; i< \lfloor \frac{m}{2} \rfloor \\
    \textnormal{HASH}(i, \lms j,k\rms) = i+1  & \; \text{if} \; j = k, \; i< \lfloor \frac{m}{2} \rfloor \\
    \textnormal{HASH}(\lfloor \frac{m}{2} \rfloor, \lms j,k \rms) = \lfloor \frac{m}{2} \rfloor & 
\end{cases},
\end{equation*}
with $j,k \leq \lfloor \frac{m}{2} \rfloor$.
Then, \textnormal{HASH} is an injective coloring over the $m$-cycle at each iteration $t$ of the 1--WL test. This gives, at each iteration $0\leq t<T = \lfloor \frac{m}{2}\rfloor$, the coloring
\begin{equation}
\label{lem:k_cycles:eq:color_formula}
    \begin{cases}
     c^{(t)}(i) = i    & \text{if} \;\; 0 \leq i \leq t    \\
      c^{(t)}(i) = t+1   & \text{if}\;\; t< i < m-t\\
      c^{(t)}(i) = m-i & \text{if} \;\; m-t \leq i < m \\
    \end{cases},
\end{equation}
and the 1-WL test terminates after $T = \lfloor \frac{m}{2} \rfloor$ iterations (i.e.,  $c^{(T)}= c^{(T-1)}$), giving  $\lfloor \frac{m}{2}\rfloor + 1 $ colors. 
\end{lemma}
\begin{proof}
We prove the lemma by induction on $t$.
\paragraph{Case $t=1$} We start with $c^{(0)}(0)=0$ and $c^{(0)}(i) = 1$, for $i = 1, \dots, m-1$. We only have three hashing cases:
    \begin{itemize}
        \item[$\circ$] $\text{HASH}(0, \lms 1,1\rms) = 0$, the color assigned to node 0;
        \item[$\circ$] $\text{HASH}(1,\lms 0,1\rms) = 1$, the color assigned to nodes $1$ and $m-1$;
        \item[$\circ$] $\text{HASH}(1,\lms 1,1\rms) = 2$, the color assigned to all nodes $1<i<m-1$.
    \end{itemize}
    This shows that $c^{(1)}$ satisfies \eqref{lem:k_cycles:eq:color_formula} and that $\text{HASH}$ is injective at iteration $t=1$. Hence, the claim is true for $t=1$.
\paragraph{Inductive step $t \rightarrow t+1$} Assume that the inductive hypothesis is true for step $t$. Hence, our coloring is of the form  \eqref{lem:k_cycles:eq:color_formula} and that $\text{HASH}$ is injective at iteration $t$.
This means that for $0 < i \leq t$   we have $c^{(t)}(i-1) < c^{(t)}(i) < c^{(t)}(i+1)$ and for  $m-t \leq i < m-1$ we have $c^{(t)}(i+1) < c^{(t)}(i) < c^{(t)}(i-1)$; thus, for $0 < i \leq t $ or $ m-t-1 \leq i < m-1$, we see that
\begin{equation*}
     c^{(t+1)}(i) = \text{HASH}(c^{(t)}(i), \lms c^{(t)}(i-1), c^{(t)}(i+1) \rms = i.
\end{equation*}
For $i=t+1$  we have $c^{(t)}(i-1) < c^{(t)}(i) = c^{(t)}(i+1)$ and for $i=m-t-2$ we have $c^{(t)}(i+1) < c^{(t)}(i) = c^{(t)}(i-1)$; therefore, for $i=t+1$ and $i=m-t-2$, we also have
\begin{equation*}
   c^{(t+1)}(i) = \text{HASH}(c^{(t)}(i), \lms c^{(t)}(i-1), c^{(t)}(i+1) \rms = i.
\end{equation*}
For all the remaining indices $t+1< i < m-t-2$, we have $c^{(t)}(i-1) = c^{(t)}(i) = c^{(t)}(i+1)$, so 
\begin{align*}
   c^{(t+1)}(i) & = \text{HASH}(c^{(t)}(i), \lms c^{(t)}(i-1), c^{(t)}(i+1) \rms\\
   & =  (t+1) +1 = t+2.
\end{align*}
The HASH function is still injective, as for $0 < i \leq t+1$ we have $c^{(t)}(i-1) < c^{(t)}(i) < c^{(t)}(i+1)$, for $m-t-1 \leq i < m-1$ we have $c^{(t)}(i+1) < c^{(t)}(i) < c^{(t)}(i-1)$, and for $t+1< i < m-t-1$  it holds $\text{HASH}(c^{(t)}(i),\lms c^{(t)}(i-1),c^{(t)}(i+1) \rms) = \text{HASH}(t+1,\lms t+1,t+1 \rms) =t+2$. Therefore, we have
\begin{equation*}
\begin{cases}
     c^{(t+1)}(i) = i    & \text{if} \;\; 0 \leq i \leq t+1   \\
      c^{(t+1)}(i) = t+2   & \text{if}\;\; t+1< i < m-t-1  \\
    c^{(t+1)}(i) = m-i & \text{if} \;\;  m- t -1 \leq i < m
\end{cases}.
\end{equation*}

\paragraph{Termination of the 1--WL test}
At iteration $\lfloor \frac{m}{2} \rfloor -1 $ we have
\begin{equation*}
    \begin{cases}
     c^{(\lfloor \frac{m}{2} \rfloor -1 )}(i) = i    & \text{if} \;\; 0 \leq i \leq \lfloor \frac{m}{2} \rfloor -1  \\    
      c^{(\lfloor \frac{m}{2} \rfloor -1)}(i) =\lfloor \frac{m}{2} \rfloor    & \text{if}\;\; i = \lfloor \frac{m}{2} \rfloor \; \text{or} \; i = \lceil \frac{m}{2} \rceil \\
      c^{(\lfloor \frac{m}{2} \rfloor -1 )}(i) = m-i   & \text{if} \;\; \lceil \frac{m}{2} \rceil +1 \leq i < m
      \end{cases}.
\end{equation*}
This concludes the proof.
\end{proof}

A graphical representation of Lemma \ref{lem:k_cycles} can be found in Figure \ref{fig:cycle_sequence}.
\begin{figure*}[t]
    \centering
    \includegraphics[width = 0.8\textwidth]{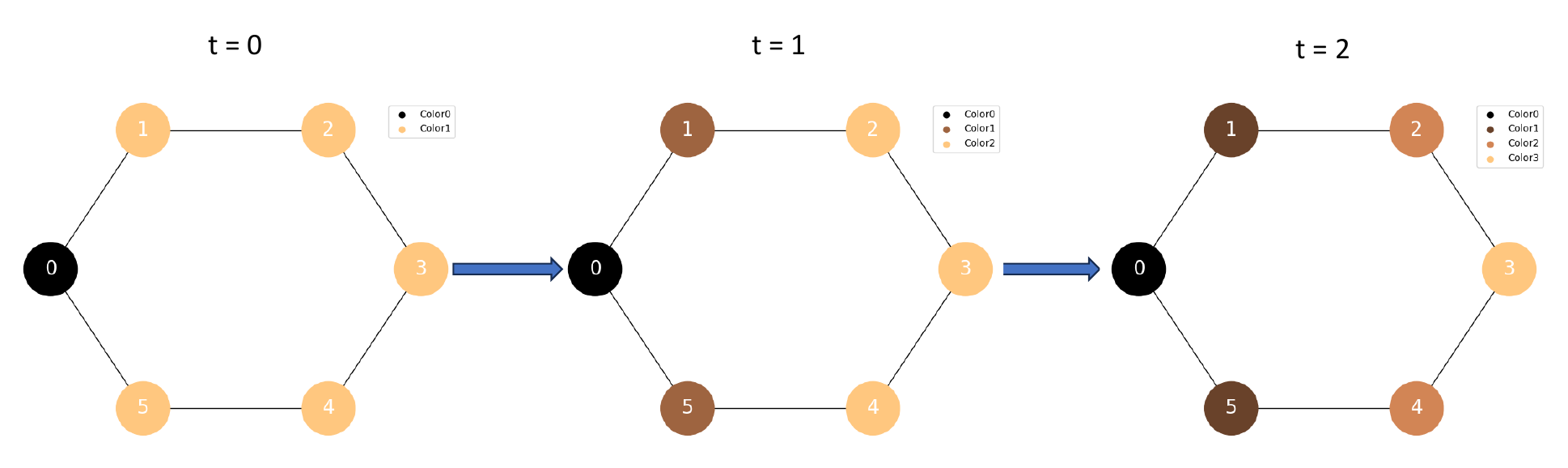}
    \caption{Graphical illustration of  Lemma~\ref{lem:k_cycles}: a 6-cycle reaches a stable coloring in $\lfloor \frac{6}{2} \rfloor = 3$ steps with $\lfloor \frac{6}{2} \rfloor +1 = 4$ colors. Numbers are used to identify nodes.}
    \label{fig:cycle_sequence}
\end{figure*}
We observe that the specific node indexing of Lemma~\ref{lem:k_cycles} was adopted just to ease  computations; nevertheless, it is possible to construct a HASH function for other choices of node indexing. This is due to the fact that the mapping depends only on the topological structure in each node's neighborhood.   
This lemma represents the core of next theorem's proof, which establishes the ability of the 1-WL test to classify dicyclic graphs with identical cycles. Intuitively, if we have a dicyclic graph where node colors are uniformly initialized, one step of 1--WL test yields a coloring depending entirely on the number of neighbours for each node. In a dicyclic graph $[m,n]$ we always have $m+n-2$ nodes of degree two and $2$ nodes of degree three, so $c^{(1)}(i)=1$ for all 2-degree nodes $i$, and $c^{(1)}(j)=0$ for the two 3-degree nodes $j$. Hence, each cycle of the dicyclic graph satisfies the initial coloring hypothesis of Lemma~\ref{lem:k_cycles}. 
\begin{theorem}[1-WL test on dicyclic graphs]\label{th:3_degrees_nodes}
    The 1--WL test gives the same color to the 3-degree nodes of a uniformly colored dicyclic graph $[m,n]$ (i.e., $c^{(0)}= 0 \in \mathbb{N}^{m+n}$) if and only if $m=n$. Therefore, the 1--WL test can classify symmetric dicyclic graphs.
\end{theorem}
\begin{proof}
After one iteration on the 1--WL test, regardless of the symmetry of the dicyclic graph, we obtain a coloring in which only 3-degree nodes have a different color, whose value we set to 0. We can therefore split the coloring vector $c^{(1)} \in \mathbb{N}^{m+n}$ in two subvectors, namely, $c^{(1)} = [(c_1^{(1)})^T, (c_2^{(1)})^T]^T$ corresponding to each cycle, respectively, and where $c_1^{(1)}(0)$ and $c_2^{(1)}(0)$ correspond to the 3-degree nodes. We treat the symmetric and the asymmetric cases separately.

\paragraph{The symmetric case} We let $c_1^{(0)}=c_2^{(0)} = c_0^{(0)}$, with $c_0^{(0)} = [0, 1, \dots, 1]$. In this case, we run the 1--WL test in parallel on both vectors $c_1^{(t)}$ and $c_2^{(t)}$, where the HASH function in Lemma \ref{lem:k_cycles} is extended on the 3-degree nodes as $\text{HASH}(0, \lms 0, j,k \rms) = 0$. Therefore, for each $t \geq 0$, 
\begin{equation*}
    c_0^{(t+1)}(0) = \text{HASH} (c_0^{(t)}(0), \lms c_0^{(t)}(0), c_0^{(t)}(1), c_0^{(t)}(m-1) \rms)  = 0.
\end{equation*} 
Thanks to Lemma \ref{lem:k_cycles} we obtain $c_1^{(\lfloor \frac{m}{2} \rfloor )} =c_2^{(\lfloor \frac{m}{2} \rfloor  )} $, which is a stable coloring for the whole graph, as the color partition is not refined anymore. 

\paragraph{The asymmetric case} Without loss of generality, we can assume $m = \text{length}(c_1^{(t)}) \neq \text{length}(c_2^{(t)}) = m + h$ for some $h>0$. We also assume for now that $m$ is odd (the case $m$ even will be briefly discussed later).
We extend the HASH function from Lemma~\ref{lem:k_cycles} to colors $j,k> \lfloor \frac{m}{2} \rfloor$. For $j > \lfloor \frac{m}{2} \rfloor $ or $ k > \lfloor \frac{m}{2} \rfloor$ we define 
\begin{equation*}
\begin{cases}
    \text{HASH}(0, \lms j,k \rms) = \infty & \\
    \text{HASH}(i,\lms j,k\rms) = \lfloor \frac{m}{2} \rfloor + i & \; \text{if} \; j \neq k, \; i \leq \lfloor \frac{m}{2} \rfloor \\
    \text{HASH}(i, \lms j,k \rms) = \lfloor \frac{m}{2} \rfloor + i+1  & \; \text{if} \; j = k, \; i \leq \lfloor \frac{m}{2} \rfloor
\end{cases} .
\end{equation*}
Running in parallel the 1--WL test on the two cycles, computing the coloring vectors $c_1^{( \lfloor \frac{m}{2} \rfloor+1)}$ and $c_2^{( \lfloor \frac{m}{2} \rfloor+1)}$ up to iteration $\lfloor \frac{m}{2} \rfloor+1$, 
for  $ i = \lfloor \frac{m}{2} \rfloor +1$  we have $c_2(i) = \lfloor \frac{m}{2} \rfloor + 1 $. Therefore, given the extension of the HASH function just provided, this new color starts to backpropagate on the indices $i< \lfloor \frac{m}{2} \rfloor +1$ , $i > m-h-\lfloor \frac{m}{2} \rfloor -1$ until it reaches the index $0$. As a consequence, it exists an iteration index $T$ such that $c_2^{(T)}(0) =  \text{HASH} (0, \lms j, k^*\rms)$ with $k^* > \lfloor \frac{m}{2} \rfloor $ and, finally, $c_2^{(T)}(0) = \infty$, giving $c_1^{(T)}(0) \neq c_2^{(T)}(0)$,
as claimed. 

The case in which $m$ is even works analogously, but we have to modify the HASH function in a different way to preserve injectivity. In particular, for $j,k \leq m/2$, we define
\begin{equation*}
    \begin{cases}
        \text{HASH}(i,\lms j,k \rms)= \frac{m}{2}  &  \; \text{if} \; j=k, i =  \frac{m}{2}  \\
        \text{HASH}(i,\lms j,k \rms)=  \frac{m}{2}  +1 &  \; \text{if} \; j\neq k, i =  \frac{m}{2}  
    \end{cases}.
\end{equation*}
This concludes the proof.
\end{proof}


Theorem~\ref{th:3_degrees_nodes} establishes in a deterministic way the power of the 1--WL test in terms of distinguishing between symmetric and asymmetric dicyclic graphs, given a sufficient number of iterations directly linked with the maximum cycle length in the considered domain. Examples of 1-WL stable colorings on dicyclic graphs are presented in Figure~\ref{fig:wl_coloring}.

\begin{figure}[t]
\centering
\begin{subfigure}{0.5\textwidth}
\centering
    \includegraphics[width=0.7\linewidth]{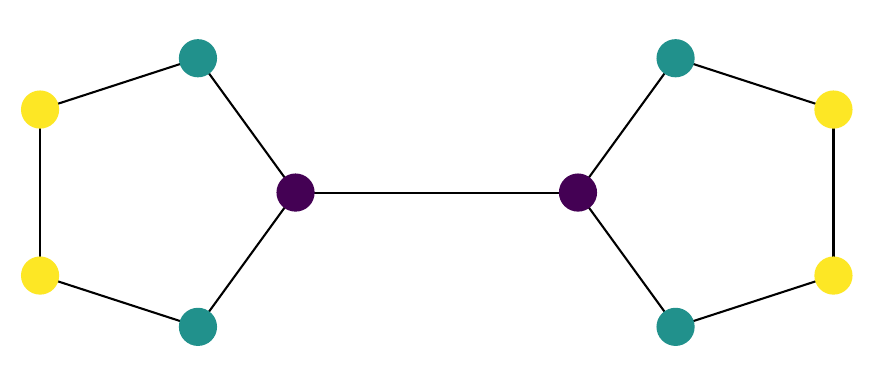}
    \caption{Stable 1-WL coloring for $[5,5]$.}
    \label{fig:5_5colored}
\end{subfigure}
\begin{subfigure}{0.5\textwidth}
\centering
    \includegraphics[width=0.7\linewidth]{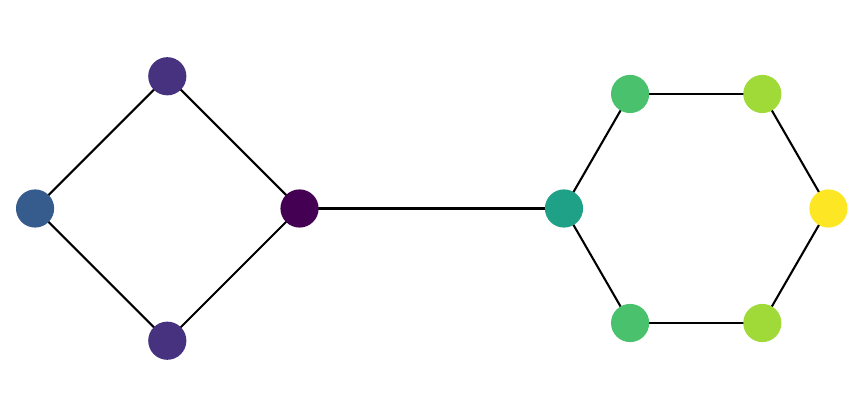}
    \caption{Stable 1-WL coloring for $[4,6]$.}
    \label{fig:4_6colored}
\end{subfigure}
\caption{Stable 1-WL coloring for different types of dicyclic graphs: as stated in Theorem \ref{th:3_degrees_nodes},  3-degree nodes have the same color in symmetric dicyclic graphs, and different color in the asymmetric ones. }
\label{fig:wl_coloring}
\end{figure}

Employing well-known results in the literature concerning the expressive power of GNNs (see \cite{morris2019weisfeiler,xu2018powerful} and in particular Theorem \ref{th:wl_morris}), we can prove the main result of this subsection on the classification power of GNNs on the domain of dicyclic graphs.

\begin{corollary}[GNNs can classify symmetric dicyclic graphs]\label{cor:gnn_topological_identity}
   There exist a GNN of the form \eqref{morris_gconv} and a \textnormal{READOUT} function able to classify symmetric dicyclic graphs.

\end{corollary}
\begin{proof}
Let $[m,n]$  be a dicyclic graph and $c^{(T)}$  be the stable coloring of $[m,n]$ produced by the 1-WL test with initial uniform coloring. By Theorem \ref{th:3_degrees_nodes} the graph can be correctly classified by the 1--WL test, i.e., by its stable coloring. 
Using Theorem \ref{th:wl_morris}, a GNN $f_{\Theta}$ exists such that  $f_{\Theta}$ can learn the stable coloring for each input graph for each iteration step $t$. Let $c^{(T)}$ be the stable coloring computed by a GNN for a dicyclic graph $[m,n]$. Let $(u,v)$ be the 3-degree nodes of the dicyclic graph. Then, the READOUT can be modeled as
\begin{equation*}
    \text{READOUT}(c^{(T)}) = \begin{cases}
        1 & \text{if} \; c^{(T)}(u)=c^{(T)}(v) \\
        0 & \text{otherwise}
    \end{cases}.
\end{equation*}
 With such a READOUT, the GNN assigns the correct rating to the dicyclic graph (i.e., 1 if the graph is symmetric, 0 otherwise).\end{proof}

\begin{remark}[The gap between theory and practice in Corollary~\ref{cor:gnn_topological_identity}]
\label{rem:gap_theory_practice_dicyclic}
    Corollary \ref{cor:gnn_topological_identity} shows that GNNs are powerful enough to match the 1--WL test's expressive power for the  classification of symmetric dicyclic graphs (as established by Theorem ~\ref{th:3_degrees_nodes}). However, it is worth underlining that this result only proves the \textit{existence} of a GNN model able to perform this task. In contrast to the results presented in \S\ref{subsec:id_eff_words}, this corollary does not mention any training procedure. Nevertheless, the numerical experiments  in \S\ref{subsec:exp_dicyclic} show that GNNs able to classify symmetric dicyclic graphs \textit{can} be trained in practice, albeit achieving generalization outside the training set is not straightforward and depends on the GNN architecture.
\end{remark}

\section{Numerical results}\label{sec:experiments}

This section presents the results of experimental tasks designed to validate our theorems. We analyze the consistency between theoretical and numerical findings, highlighting the significance of specific hypotheses, and addressing potential limitations of the theoretical results.

\subsection{Experimental Setup}
\label{subsec:exp_setup}
 We take in account two different models for our analysis:

\begin{itemize}
    \item The Global Additive Pooling GNN (\textit{Gconv-glob}) applies a sum pooling at the end of the message-passing convolutional layers \cite{hamilton2017inductive}.
    In the case of the 2-letter words setting, the resulting vector $h_{\text{glob}} \in \mathbb{R}^h$ undergoes processing by a linear layer, while in the dicyclic graphs setting, an MLP is employed. A sigmoid activation function is applied at the end.
    
    \item The Difference GNN (\textit{Gconv-diff}), takes the difference between the hidden states of the two nodes in the graph (in the 2-letter words setting) or the difference between the hidden states of the 3-degree nodes (in the dicyclic graphs setting) after the message-passing convolutional layers. The resulting vector $h_{\text{diff}} \in \mathbb{R}^h$ is then fed into a final linear layer, followed by the application of a sigmoid activation function.
\end{itemize}

The choice of the last READOUT part is driven by empirical observation on their effectiveness on the two different tasks. \\
Training is performed on an Intel(R) Core(TM) i7-9800X processor running at 3.80GHz using 31GB of RAM along with a GeForce GTX 1080 Ti GPU unit. The Python code is available at \url{https://github.com/AleDinve/gnn_identity_effects.git}.

\subsection{Case study \#1: two-letter words}\label{subsec:exp_twoletter}

To validate Theorem~\ref{th:gnn_rating}, we consider a classification task using the two-letter word identity effect problem described in \S\ref{subsubsec:id_eff_application}, following the experimental setup presented in \cite{brugiapaglia2022invariance}.
 
\subsubsection{Task and datasets}

\label{subsubsec:task_data}
 In accordance with the setting of \S\ref{subsubsec:id_eff_application}, each word is represented as a graph consisting of two nodes connected by a single unweighted and undirected edge (see Figure~\ref{fig:two-letters-graph}). Each node is assigned a node feature $x \in \mathbb{R}^{26}$, corresponding to a letter's encoding.

The training set $D_{\text{train}}$ includes all two-letter words composed of any English alphabet letters except $\mathsf{Y}$ and $\mathsf{Z}$. The test set $D_{\text{test}}$ is a set of two-letter words where at least one of the letters is chosen from ${\mathsf{Y},\mathsf{Z}}$. Specifically, we consider $D_{\text{test}} = \{ \mathsf{YY}, \mathsf{ZZ}, \mathsf{YZ}, \mathsf{ZT}, \mathsf{EY}, \mathsf{SZ} \}$.

\subsubsection{Vertex feature encodings}
In our experiments, we consider four different encodings of the English alphabet, following the framework outlined in \S\ref{subsec:id_eff_words}. Each encoding consists of a set of vectors drawn from $\mathbb{R}^{26}$.

\begin{itemize}
\item \textit{One-hot encoding}: This encoding assigns a vector from the canonical basis to each letter: $\mathsf{A}$ is encoded as $e_1$, $\mathsf{B}$ as $e_2$, ..., and $\mathsf{Z}$ as $e_{26}$.
\item \textit{Haar encoding}: This encoding assigns to each letter the columns of a $26 \times 26$ orthogonal matrix drawn from the orthogonal group $\text{O}(26)$ using the Haar distribution \cite{mezzadri2007generate}.
\item \textit{Distributed encoding}: This encoding assigns a random combination of 26 bits to each letter. In this binary encoding, only $j$ bits are set to 1, while the remaining $26-j$ bits are set to 0. In our experiments, we set $j = 6$. 
\item \textit{Gaussian encoding}: This encoding assigns samples from the multivariate normal distribution $\mathcal{N}(0, I)$, where $0 \in \mathbb{R}^n$ and $I \in \mathbb{R}^{n \times n}$. In our experiments, we set $n = 16$. 
\end{itemize}
Observe that only the one-hot and the Haar encodings are orthogonal (see \S\ref{subsubsec:id_eff_application}) and hence satisfy the assumption of Theorem~\ref{th:rating_alphabet}. On the other hand, the distributed and the Gaussian encodings do not fall within the setting of Theorem~\ref{th:rating_alphabet}.

We run 40 trials for each model (i.e., Gconv-glob or Gconv-diff, defined in \S\ref{subsec:exp_setup}) with $l$ layers (ranging from 1 to 3). In each trial, a different training set is randomly generated.
The models are trained for 5000 epochs using the Adam optimizer with a learning rate of $\lambda = 0.0025$, while minimizing the binary cross-entropy loss. The hidden state dimension is set to $d = 64$, and Rectified Linear Units (ReLUs) are used as activation functions.

The numerical results are shown in Figures \ref{fig:gconvglob}--\ref{fig:gconvdiff}, where we propose two different types of plots:
\begin{figure*}[t]
\centering
\begin{subfigure}{\linewidth}
\centering
    \includegraphics[width=0.3\linewidth]{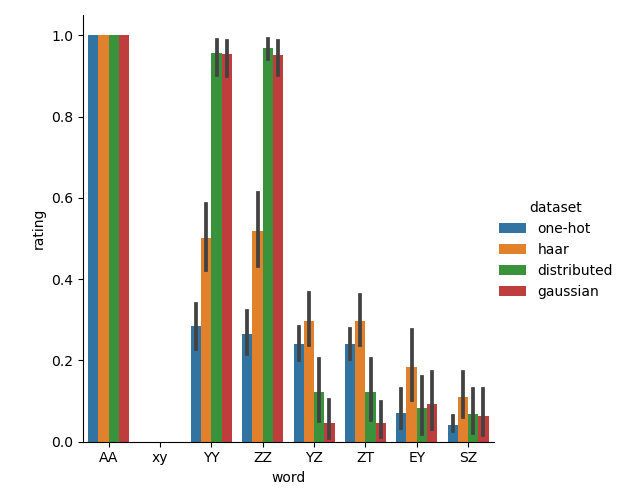}
    \includegraphics[width=0.3\linewidth]{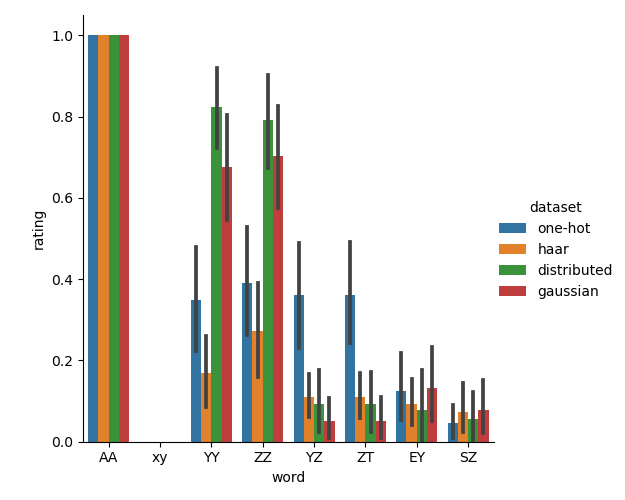}
    \includegraphics[width=0.3\linewidth]{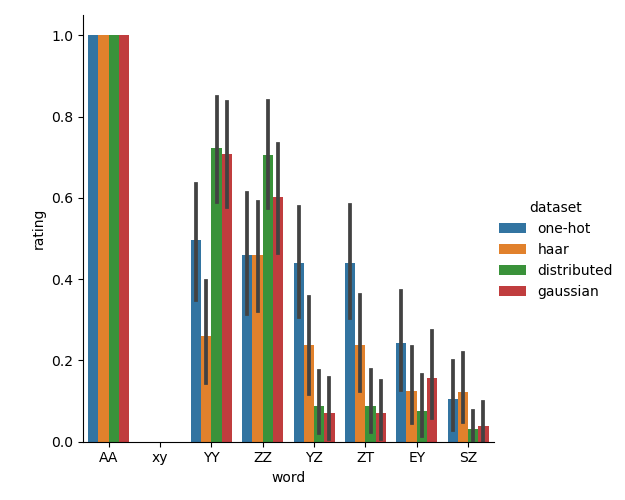}
\end{subfigure}
\begin{subfigure}{\linewidth}
\centering
    \includegraphics[width=0.3\linewidth]{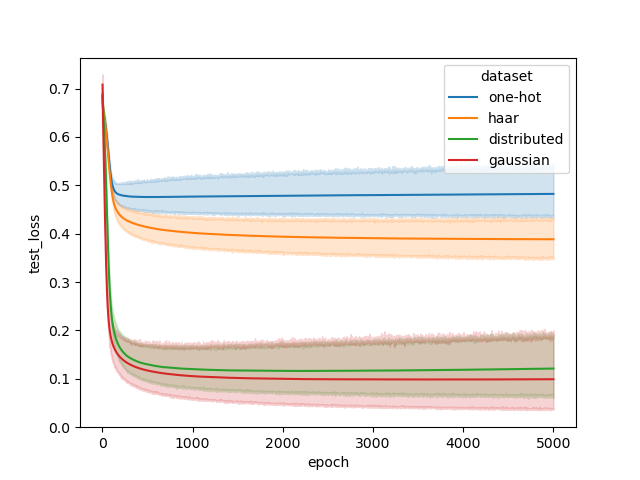}
    \includegraphics[width=0.3\linewidth]{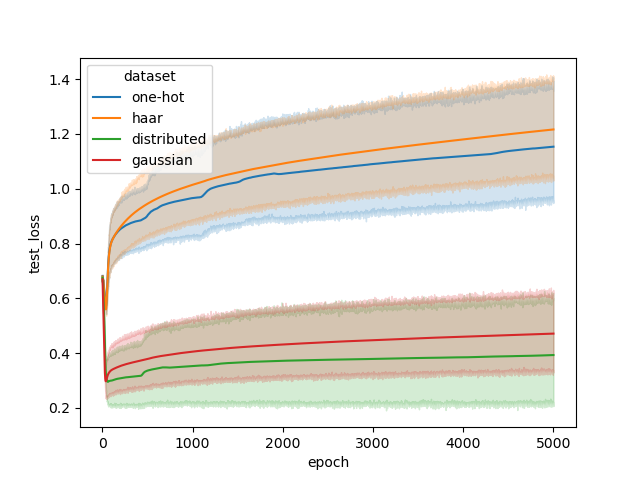}
    \includegraphics[width=0.3\linewidth]{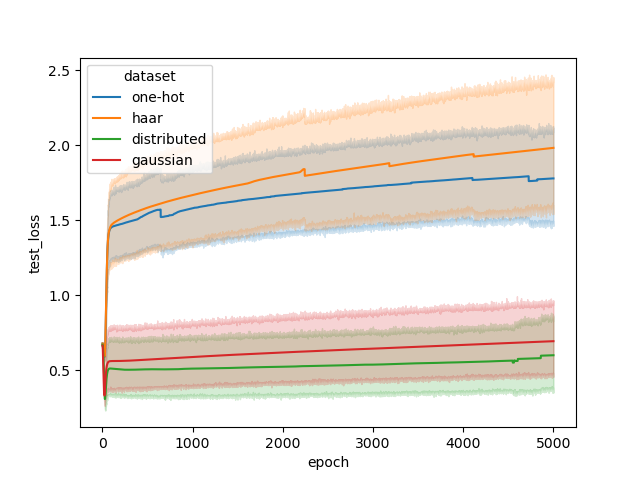}
\end{subfigure}
\caption{Numerical results for the rating task on the two-letter words dataset using Gconv-glob with $l=1,2,3$ layers. Rating should be equal to 1 if words are composed by identical letters, 0 otherwise. The distributed and Gaussian encodings, which deviate from the framework outlined in Theorem \ref{th:gnn_rating}, exhibit superior performance compared to the other encodings. The other encodings makes  the transformation matrix orthogonal and symmetric, being themselves orthogonal encodings.}
\label{fig:gconvglob}
\end{figure*}
\begin{figure*}[t]
\centering
\begin{subfigure}{\linewidth}
\centering
    \includegraphics[width=0.3\linewidth]{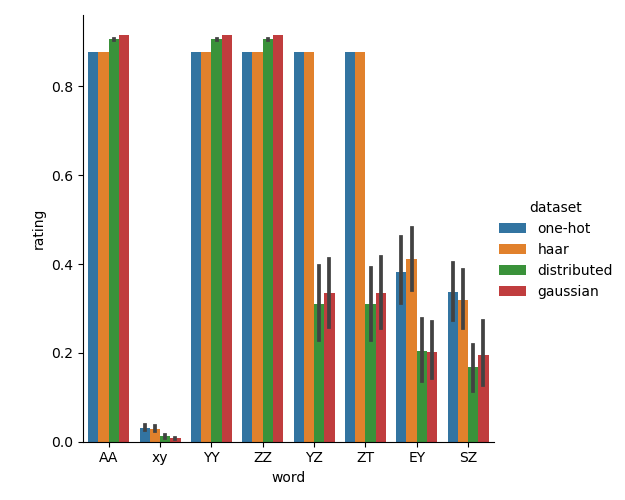}
    \includegraphics[width=0.3\linewidth]{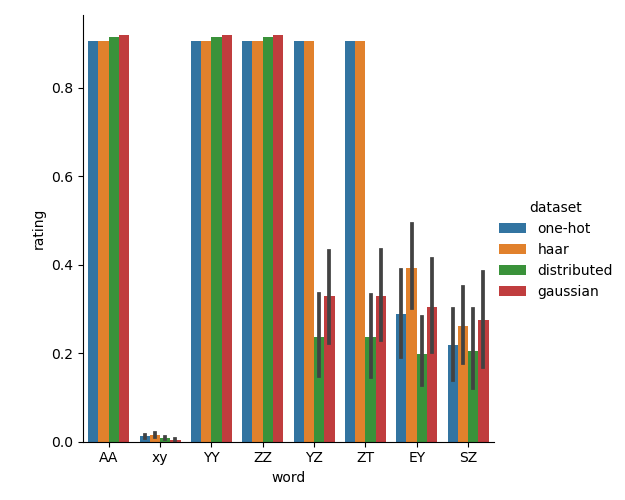}
    \includegraphics[width=0.3\linewidth]{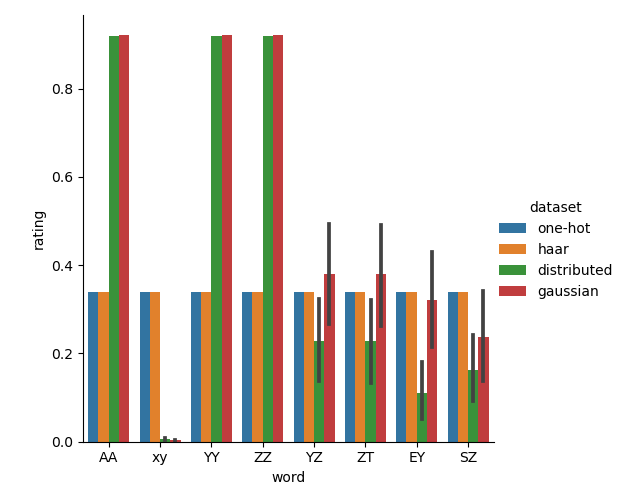}
\end{subfigure}
\begin{subfigure}{\linewidth}
    \centering
    \includegraphics[width=0.3\linewidth]{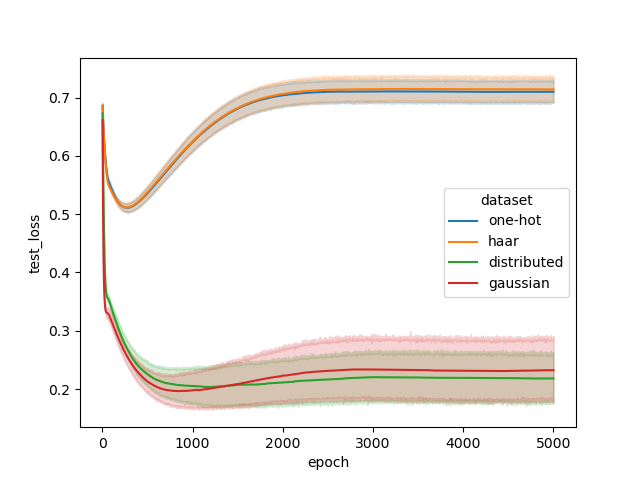}
    \includegraphics[width=0.3\linewidth]{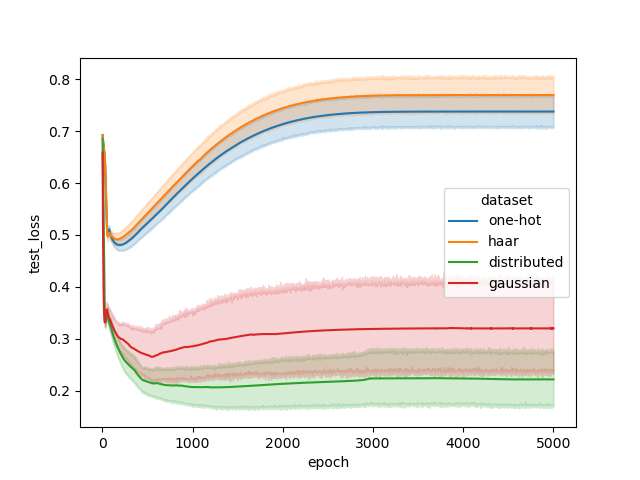}
    \includegraphics[width=0.3\linewidth]{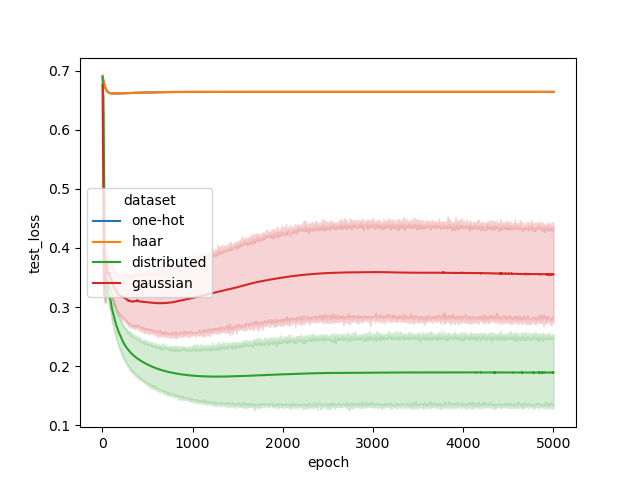}
\end{subfigure}
\caption{Numerical results for the rating task on the two-letter words dataset using Gconv-diff with $l=1,2,3$ layers. The same observations to those in Figure \ref{fig:gconvglob} can be made here as well.}
\label{fig:gconvdiff}
\end{figure*}

\begin{itemize}
\item On the top row, we compare the ratings obtained using the four adopted encodings. The first two words, $\mathsf{AA}$ and a randomly generated word with nonidentical letters, denoted $\mathsf{xy}$, are selected from the training set to showcase the training accuracy. The remaining words are taken from $D_{\text{test}}$, allowing assessment of the  generalization capabilities of the encoding scheme outside the training test. The bars represent the mean across trials, while the segments at the center of each bar represent the standard deviation.

\item On the bottom row, we show loss functions with respect to the test set over the training epochs for each encoding. The lines represent the average, while the shaded areas represents the standard deviation.
\end{itemize}

Our numerical findings indicate that the rating impossibility theorem holds true for the one-hot encoding and the Haar encoding. However, notable differences in behavior emerge for the other two encodings. The 6-bit distributed encoding exhibits superior performance across all experiments, demonstrating higher rating accuracy and better loss convergence. The Gaussian encoding yields slightly inferior results, yet still showcases some generalization capability.
It is important to note that despite variations in experimental settings such as architecture and optimizer (specifically, the use of ReLU activations and the Adam optimizer), the divergent behavior among the considered encodings remains consistent. This highlights the critical role of the transformation matrix $T_2$ within the hypothesis outlined in Theorem \ref{th:rating_alphabet}.
It is interesting to notice  that increasing the number of layers contributes to the so-called \textit{oversmoothing effect} \cite{oono2019graph,cai2020note}: many message passing iterations tend to homogenize information across the nodes, generating highly similar features.

\subsection{Case study \#2: dicyclic graphs}
\label{subsec:exp_dicyclic}

We now consider the problem of classifying unlabeled symmetric dicyclic graphs, introduced in \S\ref{subsec:id_eff_cycles}. In Corollary~\ref{cor:gnn_topological_identity} we proved the existence of GNNs able to classify symmetric dicyclic graphs. In this section, we assess whether such GNNs can be computed via training (see also Remark~\ref{rem:gap_theory_practice_dicyclic}). With this aim, we consider two experimental settings based on different choices of training and test set: an \textit{extraction task} and an \textit{extrapolation task}, summarized in Figures~\ref{fig:extraction_results} and \ref{fig:extrap_results}, respectively, and described in detail below. Each task involves running 25 trials for the Gconv-glob and Gconv-diff models defined in \S\ref{subsec:exp_setup}. The number of layers in each model is determined based on the specific task.

The models are trained over 5000 epochs using a learning rate of $\lambda = 0.001$. We employ the Adam optimizer, minimizing the binary crossentropy, and incorporate the AMSGrad fixer \cite{reddi2019convergence} to enhance training stability due to the large number of layers. Labels are all initialized uniformly as $h_v^{(0)}=1$ for each node in each graph.  The hidden state dimension is set to $d=100$, and ReLU activation functions are utilized.

The results presented in Figures \ref{fig:wl}, \ref{fig:extraction_results}, and \ref{fig:extrap_results} should be interpreted as follows: each circle represents a dicyclic graph $[m,n]$; the color of the circle corresponds to the rating, while the circle's radius represents the standard deviation.

\begin{figure}[t]
\begin{subfigure}{0.23\textwidth}
\centering
\includegraphics[width=\linewidth]{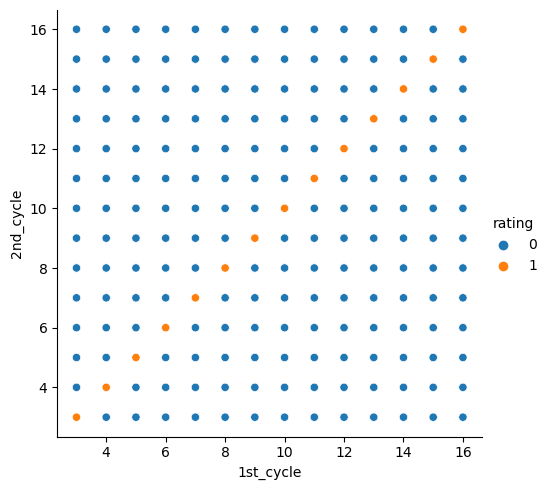}
\caption{$n_{\max}=16$}
\end{subfigure}
\begin{subfigure}{0.23\textwidth}
\centering
\includegraphics[width=\linewidth]
{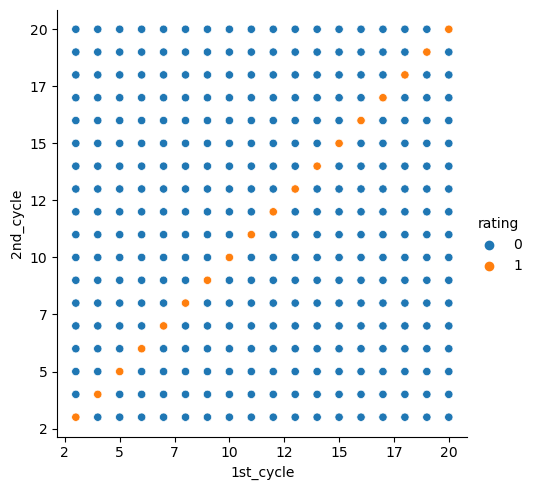}
\caption{$n_{\max}=20$}
\end{subfigure}
\caption{Perfect classification of symmetric dicyclic graphs by $n_{\max}$  iterations of the 1-WL test.}
\label{fig:wl}
\end{figure}

\subsubsection{1--WL test performance}
In Theorem~\ref{th:3_degrees_nodes} we showed that the 1--WL test can classify symmetric dicyclic graphs. This holds true regardless of the length of the longer cycle, provided that a sufficient number of iterations is performed.
The results in Figure~\ref{fig:wl} show that the 1--WL test achieves indeed perfect classification accuracy in $n_{\max}$ iterations, where $n_{\max}$ is the maximum length of a cycle in the dataset, in accordance with Theorem~\ref{th:3_degrees_nodes}.

\subsubsection{Extraction task}

In this task, we evaluate the capability of GNNs to generalize to unseen data, specifically when the minimum length of cycles in the test dataset is smaller than the maximum length of those in the training dataset. More specifically, the training set $D_{\text{train}}$ consists of pairs $[m,n]$ where $3 \leq m,n \leq n_{\max}$ and $m,n \neq k$ with $3 \leq k \leq n_{\max}$, while the test set $D_{\text{test}}$ comprises pairs $[k,a]$ with $3 \leq a \leq n_{\max}$. Figure~\ref{pic:extraction} illustrates this setting.
\begin{figure}[t]
    \centering
    \includegraphics[width = 0.43\textwidth]{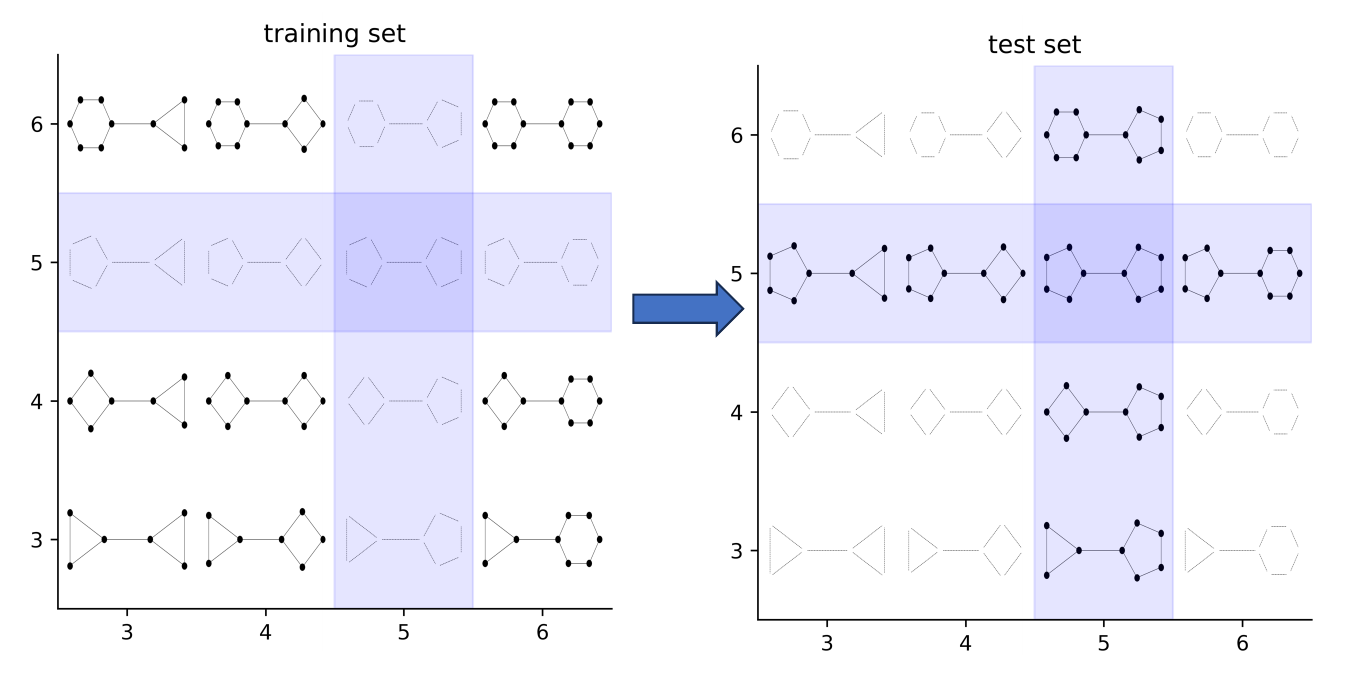}
    \caption{Graphical illustration of the extraction task. In this example, $n_{\max}=6$ and $k=5$. }
    \label{pic:extraction}
\end{figure}

In our experiments, we set $n_{\max}=8$ and consider $k$ values of 7, 6, and 5. In this setting, $|D_{\text{test}}|=(8-2)\cdot 2-1=11$ and $|D_{\text{train}}| = (8-2)^2 - |D_{\text{test}}| = 25$. The number of GNN layers is $l = n_{\max}$. The numerical results are presented in Figure \ref{fig:extraction_results}. 
\begin{figure}[t]
\centering
\begin{subfigure}{\linewidth}
\centering
\includegraphics[width=0.45\linewidth]{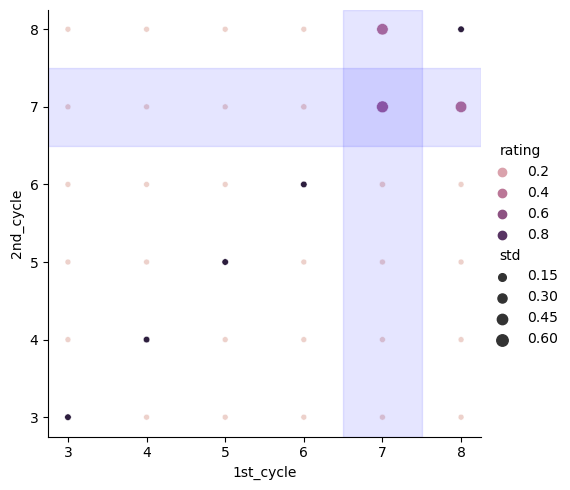}
\includegraphics[width=0.45\linewidth]{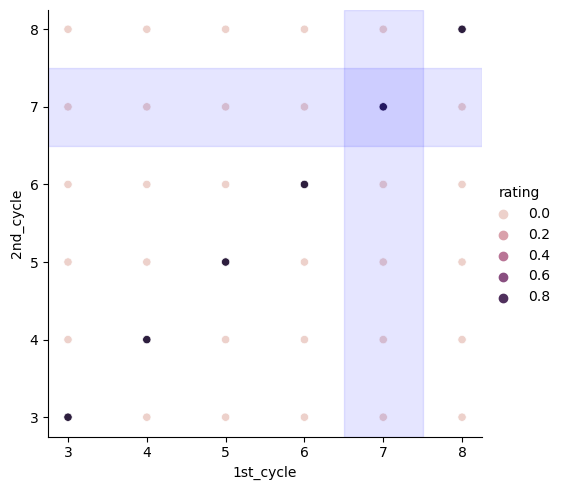}
\end{subfigure}
\begin{subfigure}{\linewidth}
\centering
\includegraphics[width=0.45\linewidth]{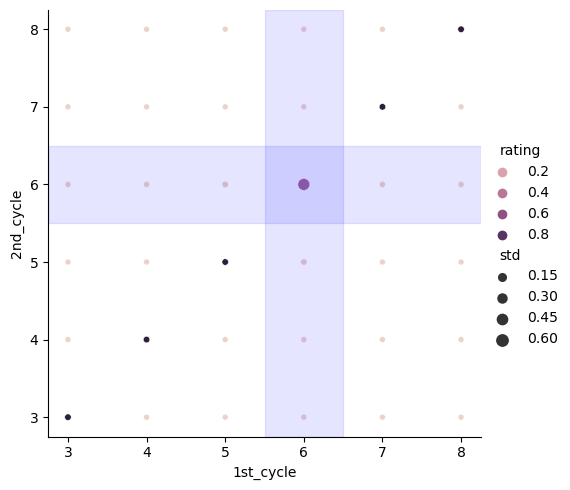}
\includegraphics[width=0.45\linewidth]{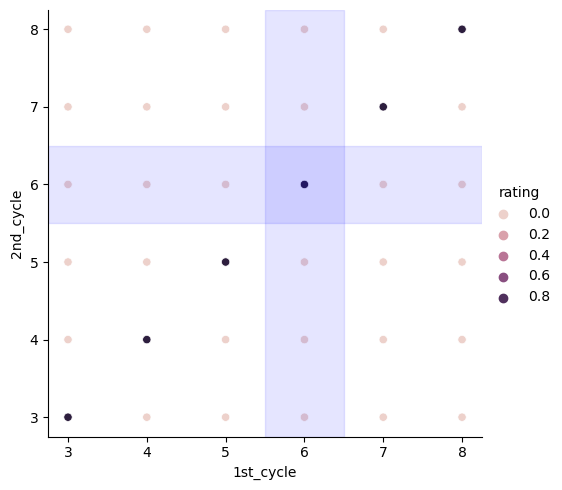}
\end{subfigure}
\begin{subfigure}{\linewidth}
\centering
\includegraphics[width=0.45\linewidth]{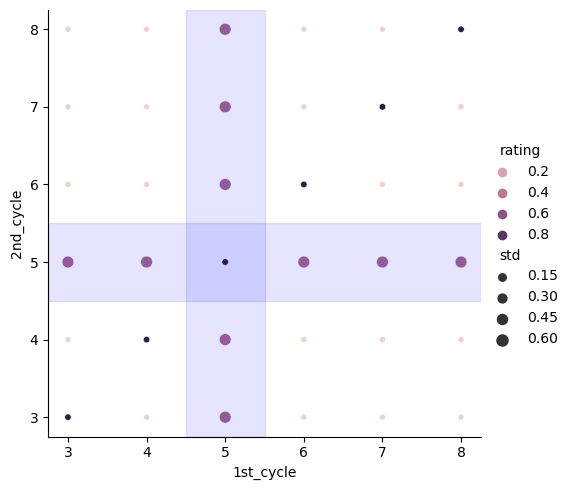}
\includegraphics[width=0.45\linewidth]{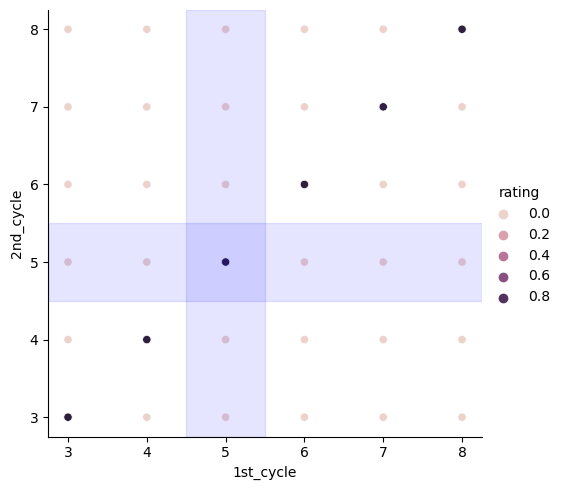}
\end{subfigure}
\caption{Extraction task performed by different GNN models, namely Gconv-glob (left) and Gconv-diff (right). We set $n_{\max}=8$, $l= 8$ and, from top to bottom, $k=7,6,5$ .}
\label{fig:extraction_results}
\end{figure}
We observe that the Gconv-diff model achieves perfect performance in our experiments (standard deviation values are not reported because they are too low), showing consistence with the theoretical setting. On the other hand, the Gconv-glob model demonstrates good, but not perfect, performance on the test set. A critical point in our numerical examples seems to be $k=5$, which falls in the middle range between the minimum and maximum cycle lengths in the training set (3 and 8, respectively). This particular value is closer to the minimum length, indicating a relatively unbalanced scenario. 

Overall, the different performance of Gconv-diff and Gconv-glob on the extraction task shows that, despite the theoretical existence result proved in Corollary~\ref{cor:gnn_topological_identity}, the choice of architecture is crucial for achieving successful generalization.

\subsubsection{Extrapolation task}
In this task, we assess GNNs' ability to generalize to unseen data with cycle lengths exceeding the maximum length in the training dataset. Specifically, the training set $D_{\text{train}}$ comprises pairs $[m,n]$ where $3 \leq m,n \leq n_{\max}$, while the test set $D_{\text{test}}$ consists of pairs $[n_{\max} + k, n']$ with $0 < k \leq g$ and $3 \leq n' \leq n_{\max} + g$. Figure~\ref{pic:extrap} illustrates the extrapolation task.
\begin{figure}[t]  
\centering \includegraphics[width=0.43\textwidth]{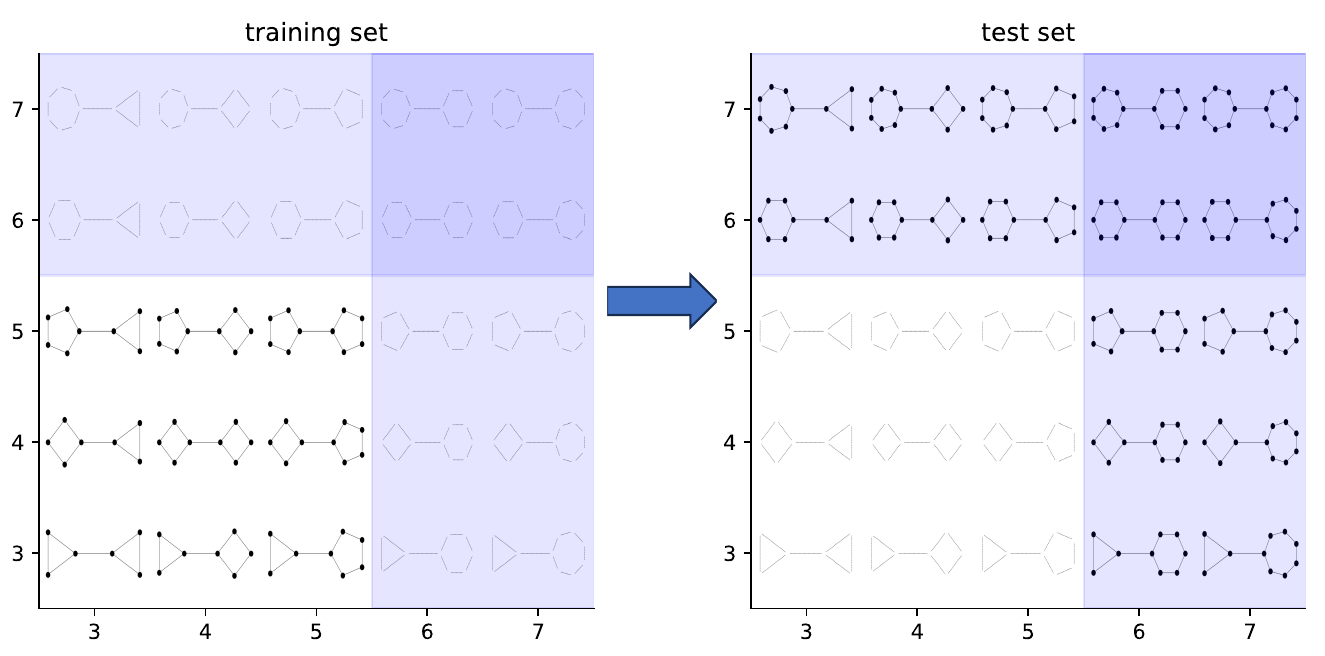}
    \caption{Graphical illustration of the extrapolation task. In this example, $n_{\max}=5$ and $g=2$. }
    \label{pic:extrap}
\end{figure}

In our experiments, we set $n_{\max}=8$ and consider values of $g$ as 1, 2, and 3. The number of GNN layers is  $l = n_{\max} + g$. Therefore, $|D_{\text{train}}| = (8-2)^2 = 36$, $|D_{\text{test},g=1}| = (9-2)\cdot2-1 = 13 $, $|D_{\text{test},g=2}| = (10-2)\cdot 4-4 = 28$ and $|D_{\text{test},g=3}| = (11-2)\cdot6-9= 45 $. Numerical results are presented in Figure~\ref{fig:extrap_results}. 
\begin{figure}[t]
\centering
\begin{subfigure}{\linewidth}
\centering
\includegraphics[width=0.45\linewidth]{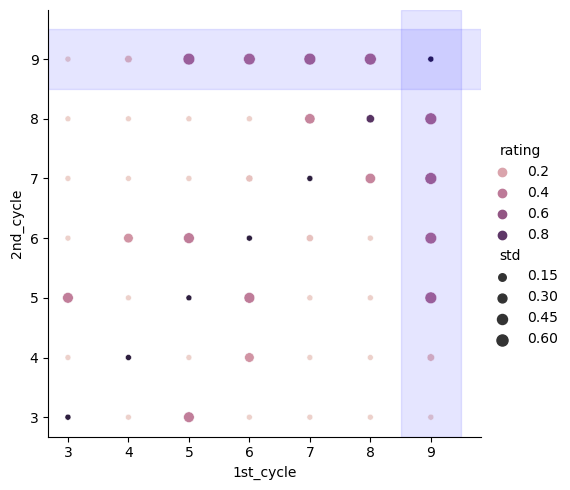}
\includegraphics[width=0.45\linewidth]{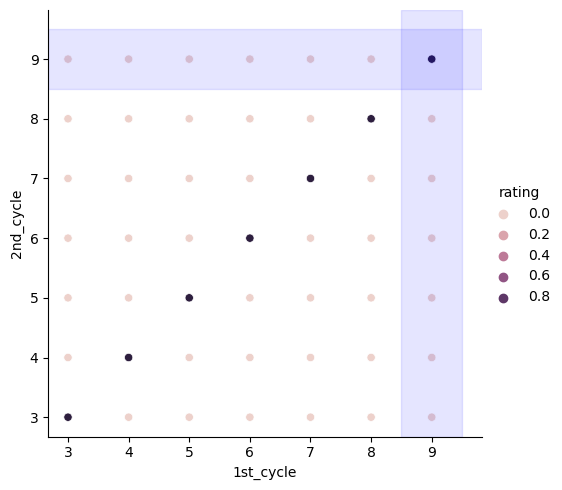}
\end{subfigure}
\begin{subfigure}{\linewidth}
\centering
\includegraphics[width=0.45\linewidth]{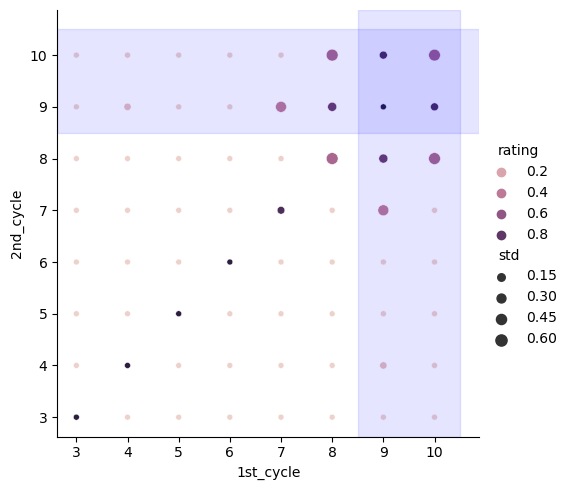} 
\includegraphics[width=0.45\linewidth]{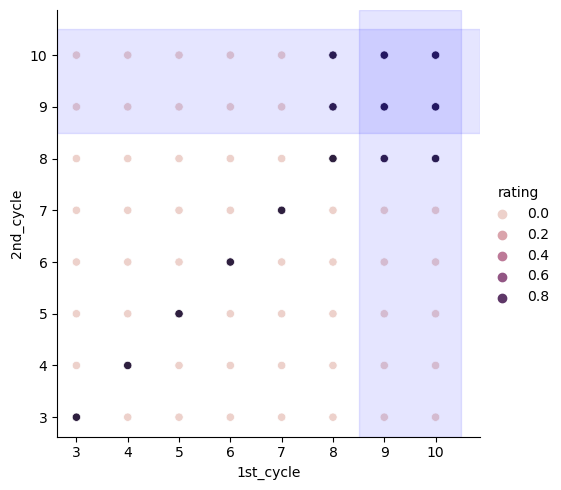}
\end{subfigure}
\begin{subfigure}{\linewidth}
\centering
\includegraphics[width=0.45\linewidth]{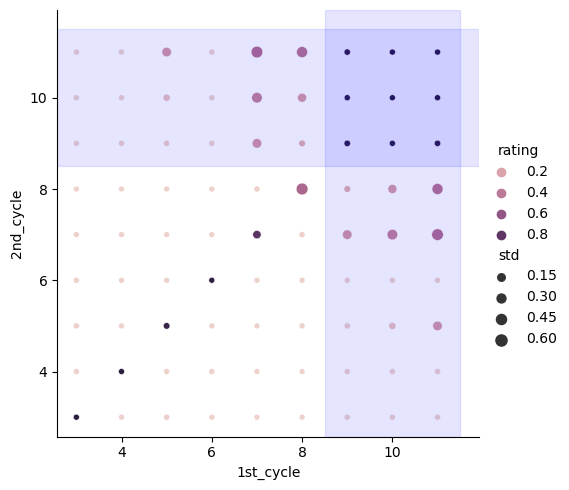}
\includegraphics[width=0.45\linewidth]{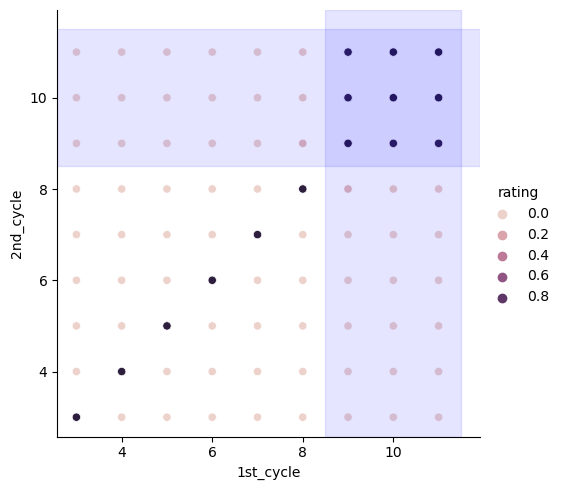}
\end{subfigure}
\caption{Extrapolation task performed by different GNN models, namely Gconv-glob (left) and Gconv-diff (right). We set $n_{\max}=8$, $l= 8$ and, from top to bottom, $(l,g) = (9,1), (10,2), (11,3)$ .}
\label{fig:extrap_results}
\end{figure}
In the extraction task, both models achieved perfect training accuracy. Conversely, in the extrapolation task, the Gconv-glob model struggles to classify the training set accurately, especially when the number of layers is equal to 9. This behavior may be attributed to the homogeneous nature of sum pooling at the end of the message passing, as it does not take into account the role of 3-degree nodes (which play a key role in our theory, as illustrated by Theorem~\ref{th:3_degrees_nodes} and Corollary~\ref{cor:gnn_topological_identity}).

On the other hand, the Gconv-diff model consistently achieves perfect training accuracy over the training set and achieves perfect generalization for $g=1$, showing once again the importance of architecture choice in practice. However, when $g \geq 2$ there is a noticeable region of misclassification for pairs $[m,n]$ where $m,n \geq n_{\max}$. This behavior could be explained by the limited capacity of the hidden states, but the optimization process might also play a significant role. 
Moreover, for $g \geq 2$ the numerical results of the extrapolation task resemble the rating impossibility phenomenon observed in the two-letter words framework. However, it is important to note that, at least for the Gconv-diff model, we observe significantly different ratings between graphs $[m,n_{\max}+g]$ where $m < n_{\max}$ and graphs $[n_{\max}+i,n_{\max}+j]$ with $i,j > 0$. In contrast, in the two-letter words framework ratings typically do not exhibit such a consistent and distinguishable pattern.

\section{Conclusions}\label{sec:conclusions}


This work extensively investigates the generalization capabilities of GNNs when learning identity effects through a combination of theoretical and experimental analysis.
From the theoretical perspective, in Theorem~\ref{th:rating_alphabet} we established that GNNs, under mild assumptions, cannot learn identity effects when orthogonal encodings are used in a specific two-letter word classification task. On the positive side, in Corollary~\ref{cor:gnn_topological_identity} we showed the existence of GNNs able to successfully learn identity effects on dicyclic graphs, thanks to the expressive power of the Weisfeiler-Lehman test (see Theorem~\ref{th:3_degrees_nodes}). 
The experimental results strongly support these theoretical findings and provide valuable insights into the problem. In the case of two-letter words, our experiments highlight the key influence of encoding orthogonality on misclassification behavior. Our experiments on dicyclic graphs demonstrate the importance of architecture choice in order to achieve generalization.

Several directions of future research naturally stem from our work. First, while Theorem~\ref{th:rating_alphabet} identifies sufficient conditions for rating impossibility, it is not known whether (any of) these conditions are also necessary. Moreover, numerical experiments on two-letter words show that generalization outside the training set is possible when using nonorthogonal encodings; justifying this phenomenon from a theoretical perspective is an open problem. On the other hand, our numerical experiments on dicyclic graphs show that achieving generalization depends on choice of the architecture; this suggests that rating impossibility theorems might hold under suitable conditions on the GNN architecture in that setting. 
Another interesting open problem is the evaluation of GNNs' expressive power on more complex graph domains. In particular, conducting extensive experiments on molecule analyses mentioned in \S\ref{sec:intro}, which naturally exhibit intricate structures, could provide valuable insights into modern chemistry and drug discovery applications.

\section*{Acknowledgments}
 GAD is partially supported by the Gruppo Nazionale per il Calcolo Scientifico (GNCS) of the Istituto Nazionale di Alta Matematica (INdAM).
SB acknowledges the support of the
Natural Sciences and Engineering Research Council of Canada (NSERC) through grant RGPIN-2020-06766 and the Fonds de Recherche du Qu\'ebec Nature et Technologies (FRQNT) through grant 313276. 
MR acknowledges the support of NSERC through grant 263344. The authors are thankful to Aaron Berk for providing feedback on a preliminary version of this manuscript and to Kara Hughes for fruitful discussions on applications to the molecule domain.





\bibliographystyle{model1-num-names}

\bibliography{references} 
\printcredits



\end{document}